%% file: main.tex
\documentclass{article}

\usepackage[final]{neurips_2024}

\usepackage[utf8]{inputenc} 
\usepackage[T1]{fontenc}    
\usepackage{booktabs}       
\usepackage{amsfonts}       
\usepackage{nicefrac}       
\usepackage{microtype}      
\usepackage{xcolor}         

\input{preamble}

\title{\neorl: Efficient Exploration for Nonepisodic RL}

\author{Bhavya Sukhija\thanks{Correspondence to \texttt{sukhijab@ethz.ch}}\ , Lenart Treven, Florian Dörfler, Stelian Coros, Andreas Krause \\
ETH Zurich, Switzerland
}

\begin{document}

\maketitle

\begin{abstract}
\looseness=-1
We study the problem of nonepisodic reinforcement learning (RL) for nonlinear dynamical systems, where the system dynamics are unknown and the RL agent has to learn from a single trajectory, i.e., adapt online and without resets. 
This setting is ubiquitous in the real world, where resetting is impossible or requires human intervention.
We propose {\em \textbf{N}on\textbf{e}pisodic \textbf{O}ptimistic \textbf{RL} (\neorl)}, an approach based on the principle of optimism in the face of uncertainty. \neorl uses well-calibrated probabilistic models and plans optimistically \wrt the epistemic uncertainty about the unknown dynamics. Under continuity and bounded energy assumptions on the system, we
provide a first-of-its-kind regret bound of $\setO(\Gamma_T \sqrt{T})$ for general nonlinear systems with Gaussian process dynamics. We compare \neorl to other baselines on several deep RL environments and empirically demonstrate that \neorl achieves the optimal average cost while incurring the least regret.
\end{abstract}

\input{mainmatter/01_introduction}
\input{mainmatter/011_problem_setting}
\input{mainmatter/02_algorithm}

\input{mainmatter/03_experiments}
\input{mainmatter/04_related_work}

\input{mainmatter/05_conclusion}

\begin{ack}
We would like to thank Mohammad Reza Karimi, Scott Sussex, and Armin Lederer for the insightful discussions and feedback on this work.
This project has received funding from the Swiss National Science Foundation under NCCR Automation, grant agreement 51NF40 180545, and the Microsoft Swiss Joint Research Center.
\end{ack}


{\small\bibliography{sources}}
\bibliographystyle{icml2024} 


\newpage
\appendix
\section*{\LARGE Appendices}

\input{backmatter/doubling_determinant_proof}
\clearpage
\input{backmatter/A_proofs}
\clearpage
\input{backmatter/experimental_details}

\end{document}

%% file: preamble.tex
\usepackage{microtype}
\usepackage{graphicx}
\usepackage{subfigure}
\usepackage{booktabs} 
\usepackage{xspace}
\usepackage{titletoc}
\usepackage{enumitem}
\usepackage{xfrac}
\usepackage{wrapfig}
\usepackage{bm}
\usepackage{algorithm, algorithmic, caption}
\usepackage{dsfont}
\usepackage{pifont}

\definecolor{chaptercolor}{HTML}{1A254B}
\definecolor{darkblue}{HTML}{1A254B}
\definecolor{linkcolor}{HTML}{2B50AA}
\definecolor{citecolor}{HTML}{2B50AA}
\definecolor{lightlinkcolor}{HTML}{9A8F97}
\definecolor{darklinkcolor}{HTML}{1A254B}
\definecolor{pink}{HTML}{E05F60}
\definecolor{lightblue}{HTML}{A7BED3}
\definecolor{red}{HTML}{F2545B}
\definecolor{blue}{HTML}{2b50aa}

\usepackage[colorlinks,linkcolor=linkcolor,citecolor=citecolor]{hyperref}
\usepackage{url}

\usepackage{amsmath}
\usepackage{amssymb}
\usepackage{mathtools}
\usepackage{amsthm}
\usepackage{bbm}
\usepackage{nicefrac}
\usepackage[capitalize,noabbrev]{cleveref}

\usepackage[para,online,flushleft]{threeparttable}
\usepackage{adjustbox}
\theoremstyle{plain}
\newtheorem{theorem}{Theorem}[section]

\newtheorem{lemma}[theorem]{Lemma}
\newtheorem{corollary}[theorem]{Corollary}
\theoremstyle{definition}
\newtheorem{definition}[theorem]{Definition}
\newtheorem{assumption}[theorem]{Assumption}

\theoremstyle{remark}

\Crefname{assumption}{Assumption}{Assumptions}
\crefname{assumption}{Assumption}{Assumptions}
\allowdisplaybreaks

\makeatletter
\renewcommand{\paragraph}{%
  \@startsection{paragraph}{4}%
  {\z@}{0ex \@plus 0ex \@minus 0ex}{-1em}%
  {\normalfont\normalsize\bfseries}
}
\makeatother

\newcommand{\inner}[2]{\left\langle#1, #2\right\rangle}
\DeclarePairedDelimiter{\ceil}{\lceil}{\rceil}

\newlist{lemenum}{enumerate}{1} 
\setlist[lemenum]{label=(\roman*),ref=\thelemma\,(\roman*),topsep=0pt}
\Crefname{lemenumi}{Lemma}{Lemmas}

\newlist{corenum}{enumerate}{1} 
\setlist[corenum]{label=(\roman*),ref=\thecorollary\,(\roman*),topsep=0pt}
\Crefname{corenumi}{Corollary}{Corollaries}

\newlist{thmenum}{enumerate}{1} 
\setlist[thmenum]{label=(\roman*),ref=\thetheorem\,(\roman*),topsep=0pt}
\Crefname{thmenumi}{Theorem}{Theorems}

\newlist{propenum}{enumerate}{1} 
\setlist[propenum]{label=(\roman*),ref=\thedefinition\,(\roman*),topsep=0pt}
\Crefname{propenumi}{Property}{Properties}

\newlist{assenum}{enumerate}{1} 
\setlist[assenum]{label=(\roman*),ref=\theassumption\,(\roman*),topsep=0pt}
\Crefname{assenumi}{Assumption}{Assumptions}

\usepackage{svg}
\usepackage{graphicx}
\usepackage{rotating}
\usepackage{tikz}
\usepackage{pgfplots}
\pgfplotsset{compat=newest}
\usetikzlibrary{shapes.geometric}

\usepackage{import}
\usepackage{xifthen}
\usepackage{pdfpages}
\usepackage{transparent}

\NewDocumentCommand{\incfig}{mo}{
  \begin{center}
    \IfValueT{#2}{\def\svgwidth{#2}}{\def\svgwidth{\columnwidth}}
    \import{./figures/}{#1.pdf_tex}
  \end{center}
}

\usepackage{pgf}
\usepackage{adjustbox}

\NewDocumentCommand{\incplt}{O{\columnwidth}m}{%
  \begin{center}
    \adjustbox{width=#1}{\import{./plots/output/}{#2.pgf}}
  \end{center}
}



\usepackage{aligned-overset}

\newcommand{\neorl}{\textcolor{black}{\textsc{NeoRL}}\xspace}

\DeclareFontFamily{U}{mathb}{\hyphenchar\font45}
\DeclareFontShape{U}{mathb}{m}{n}{
      <5> <6> <7> <8> <9> <10> gen * mathb
      <10.95> mathb10 <12> <14.4> <17.28> <20.74> <24.88> mathb12
      }{}
\DeclareSymbolFont{mathb}{U}{mathb}{m}{n}
\DeclareFontSubstitution{U}{mathb}{m}{n}
\DeclareMathSymbol{\Asterisk}      {2}{mathb}{"06}
\newcommand{\mainsym}{\textcolor{blue}{\parentheses*{\hspace{-0.1em}\Asterisk\hspace{-0.1em}}}}

\newcommand{\psym}{\ensuremath{\textcolor{blue}{\dagger}}\xspace}

\newcommand*{\abs}[1]{| #1 |}

\NewDocumentCommand{\norm}{sm}{\IfBooleanTF{#1}{\|#2\|}{\left\| #2 \right\|}}

\DeclareMathOperator*{\defeq}{\smash{\overset{\mathrm{def}}{=}}}

\DeclareMathOperator*{\argmax}{arg\,max}
\DeclareMathOperator*{\argmin}{arg\,min}

\DeclarePairedDelimiter\parentheses{(}{)}
\DeclarePairedDelimiter\brackets{[}{]}

\newcommand{\R}{\mathbb{R}}
\newcommand{\E}{\mathbb{E}}
\newcommand{\Rzero}{\mathbb{R}_{\geq 0}}

\renewcommand{\vec}[1]{{\bm{#1}}}
\newcommand{\mat}[1]{\bm{#1}}

\newcommand{\set}[1]{#1}

\NewDocumentCommand{\fnPr}{}{\mathbb{P}}
\RenewDocumentCommand{\Pr}{om}{\fnPr\IfValueT{#1}{_{#1}}\parentheses*{#2}}
\NewDocumentCommand{\Prsm}{om}{\fnPr\IfValueT{#1}{_{#1}}\parentheses{#2}}
\RenewDocumentCommand{\H}{mo}{\mathrm{H}\IfValueTF{#2}{\!\left[#1\ \middle|\ #2\right]}{\brackets*{#1}}}
\NewDocumentCommand{\Hsm}{mo}{\mathrm{H}\IfValueTF{#2}{[#1 \mid #2]}{\brackets{#1}}}
\NewDocumentCommand{\I}{mmo}{\mathrm{I}\IfValueTF{#3}{\!\left(#1;#2\ \middle|\ #3\right)}{\parentheses*{#1; #2}}}
\NewDocumentCommand{\Ism}{mmo}{\mathrm{I}\IfValueTF{#3}{(#1;#2 \mid #3)}{\parentheses{#1; #2}}}

\NewDocumentCommand{\ExpVal}{somo}{\ensuremath{\mathbb{E}\IfValueT{#2}{_{#2}}{} \IfBooleanTF{#1}{#3}{\IfValueTF{#4}{\!\left[#3\ \middle|\ #4\right]}{\brackets*{#3}}}}}
\NewDocumentCommand{\Esm}{somo}{\ensuremath{\mathbb{E}\IfValueT{#2}{_{#2}}{} \IfBooleanTF{#1}{#3}{\IfValueTF{#4}{\!\left[#3\ \middle|\ #4\right]}{\brackets{#3}}}}}
\NewDocumentCommand{\Var}{somo}{\mathrm{Var}\IfValueT{#2}{_{#2}}{} \IfBooleanTF{#1}{#3}{\IfValueTF{#4}{\!\left[#3\ \middle|\ #4\right]}{\brackets*{#3}}}}
\NewDocumentCommand{\Varsm}{somo}{\mathrm{Var}\IfValueT{#2}{_{#2}}{} \IfBooleanTF{#1}{#3}{\IfValueTF{#4}{\left[#3\ \middle|\ #4\right]}{\brackets{#3}}}}
\NewDocumentCommand{\Cov}{som}{\mathrm{Cov}\IfValueT{#2}{_{#2}}{} \IfBooleanTF{#1}{#3}{\brackets*{#3}}}
\NewDocumentCommand{\Cor}{som}{\mathrm{Cor}\IfValueT{#2}{_{#2}}{} \IfBooleanTF{#1}{#3}{\brackets*{#3}}}

\NewDocumentCommand{\grad}{e_}{\bm{\nabla}\IfValueT{#1}{_{\!\!#1}\,}}

\RenewDocumentCommand{\det}{m}{\left| #1 \right|}

\NewDocumentCommand{\tr}{m}{\mathrm{tr}\;#1}
\NewDocumentCommand{\diag}{som}{\mathrm{diag}\IfValueT{#2}{_{#2}}{}\,#3}

\NewDocumentCommand{\N}{somm}{\mathcal{N}\IfBooleanTF{#1}{\left(}{(}\IfValueT{#2}{#2;}{} #3, #4\IfBooleanTF{#1}{\right)}{)}}
\NewDocumentCommand{\GP}{omm}{\mathcal{GP}(\IfValueT{#1}{#1;}{} #2, #3)}

\DeclareMathOperator{\determinant}{det}

\newcommand{\wrt}{w.r.t.~}
\newcommand{\vzero}{\vec{0}}

\newcommand{\vf}{\vec{f}}

\newcommand{\vk}{\vec{k}}

\newcommand{\vm}{\vec{m}}

\newcommand{\vu}{\vec{u}}

\newcommand{\vw}{\vec{w}}

\newcommand{\vx}{\vec{x}}
\newcommand{\vhx}{\hat{\vec{x}}}

\newcommand{\vy}{\vec{y}}

\newcommand{\vz}{\vec{z}}

\newcommand{\veta}{\bm{\eta}}

\newcommand{\vmu}{\bm{\mu}}

\newcommand{\vphi}{\bm{\phi}}
\newcommand{\vpi}{\bm{\pi}}

\newcommand{\vsigma}{\bm{\sigma}}

\newcommand{\mI}{\mat{I}}

\newcommand{\mK}{\mat{K}}

\newcommand{\mM}{\mat{M}}

\newcommand{\mV}{\mat{V}}

\newcommand{\mX}{\mat{X}}

\def\setA{{\mathcal{A}}}
\def\setB{{\mathcal{B}}}
\def\setC{{\mathcal{C}}}
\def\setD{{\mathcal{D}}}

\def\setH{{\mathcal{H}}}

\def\setK{{\mathcal{K}}}

\def\setM{{\mathcal{M}}}
\def\setN{{\mathcal{N}}}
\def\setO{{\mathcal{O}}}

\def\setR{{\mathcal{R}}}

\def\setU{{\mathcal{U}}}

\def\setW{{\mathcal{W}}}
\def\setX{{\mathcal{X}}}

\def\setZ{{\mathcal{Z}}}

%% file: mainmatter/01_introduction.tex
\section{Introduction}\label{sec: introduction}
In recent years, data-driven control approaches, such as reinforcement learning (RL), have demonstrated remarkable achievements. However, most RL algorithms are devised for an episodic setting, where during each episode, 
the agent interacts in the environment for a predetermined episode length or until a termination condition is met. After the episode, the agent is reset back to an initial state from where the next episode commences. 
Episodes prevent the system from blowing up, i.e., maintain stability, while also restricting exploration to states that are relevant to the task at hand. Moreover, resets ensure that the agent explores close to the initial states and does not end up at undesirable parts of the state space that exhibit low reward.  In simulation, resetting is typically straightforward.  However, if we wish to enable agents to learn and adapt by interacting online with the real world, 
resets are often prohibitive since they typically involve manual intervention. Instead, agents should be able to 
learn autonomously~\citep{sharma2021autonomous} 
i.e., from a single trajectory. 
This problem is extensively studied in adaptive control~\citep{adaptivecontrol}, where classical works focus on controller design~\citep{lai1982least, lai1987asymptotically, krstic1992adaptive, krstic1995nonlinear, annaswamy2023adaptive} and not on the exploration/learning aspect of the problem. Only a few works consider these two aspects jointly~\citep{abbasi2011regret, cohen2019learning, dean2020sample, simchowitz2020naive,zhao2024data}. However, these works study linear systems with quadratic costs, i.e., the LQR setting. 
While several works in the deep RL community have also studied this problem,
(c.f.,~\cref{sec:related_work}), the theoretical results for this setting are fairly limited. In particular, theoretical results mostly exist for the finite state and action spaces~\citep{kearns2002near,
brafman2002r, jaksch10a} and  
the extension to nonlinear systems with continuous spaces is much less understood. In our work, 
we address this gap and propose a practical RL algorithm that is grounded in theory. In particular, we make the following contributions.
\paragraph{Contributions}
\begin{enumerate}[leftmargin=0.5cm]
 \item We propose, \neorl, a novel model-based RL algorithm based on the principle of optimism in the face of uncertainty. \neorl operates in a nonepisodic setting and picks
 average cost optimal policies optimistically w.r.t.~to the model's epistemic uncertainty. 
    \item We show that when the dynamics lies in a reproducing kernel Hilbert space (RKHS) of kernel $k$, \neorl exhibits a regret of $\setO(\Gamma_T \sqrt{T})$, where the regret, akin to prior work, is measured w.r.t~to the optimal average cost under known dynamics, $T$ is the number of environment steps, and $\Gamma_T$ the maximum information gain of kernel $k$~\citep{srinivas}. Our regret bound is similar to the ones obtained in the episodic setting~\citep{kakade2020information, curi2020efficient, sukhija2024optimistic, treven2024ocorl} and Gaussian process (GP) bandit optimization~\citep{srinivas, chowdhury2017kernelized, scarlett17a} and is sublinear for common kernel such as the exponential kernel. To the best of our knowledge, we are the first to obtain regret bounds for the setting.
    \item \looseness=-1
    We evaluate \neorl on several RL benchmarks against common model-based RL baselines. 
        Our experimental results demonstrate that \neorl consistently achieves sublinear regret, also when neural networks are employed instead of GPs for modeling dynamics. Moreover, in all our experiments, \neorl converges to the optimal average cost.
\end{enumerate}

%% file: mainmatter/011_problem_setting.tex
\section{Problem Setting}\label{sec:problem setting}
We consider a discrete-time dynamical system with running costs $c$.
\begin{align}
    \vx_{t+1} = \vf^*(\vx_t, \vu_t) + \vw_t \label{eq:dynamics}, \ &(\vx_t, \vu_t) \in \setX \times \setU, \ \vx(0) = \vx_0 \\
    &c(\vx, \vu) \in \R_{\geq 0} \tag{Running cost}
\end{align}
 Here $\vx_t \in \setX \subseteq \R^{d_\vx}$ is the state, $\vu_t \in \setU \subseteq \R^{d_\vu}$ the control input, and $\vw_t \in \setW \subseteq \R^{\vw}$ the process noise. The dynamics $\vf^*$ are unknown and the cost $c$ is assumed to be known.

\paragraph{Task} In this work, we study the average cost RL problem~\citep{puterman2014markov}, i.e., we want to learn the
solution to the following minimization problem
\begin{equation}
A(\vpi^{*}, \vx_0) = \min_{\vpi \in \Pi} A(\vpi, \vx_0) = \min_{\vpi \in \Pi} \limsup_{T \to \infty} \frac{1}{T}\E_{\vpi} \left[ \sum^{T-1}_{t=0} c(\vx_t, \vu_t) \right].
  \label{eq:average cost formulation}
\end{equation}
Moreover, we consider the nonepisodic RL setting where the system starts at an initial state $\vx_0 \in \setX$ but never resets back during learning, that is, we seek to learn online from a single trajectory. After each step $t$ in the environment, the RL system receives a transition tuple $(\vx_t, \vu_t, \vx_{t+1})$ and updates its policy based on the data $\setD_t$ collected thus far during learning. 
The average cost formulation is common for the nonepisodic setting~\citep{jaksch10a, 
abbasi2011regret, cohen2019learning, dean2020sample,
simchowitz2020naive}, and the cumulative regret for the learning algorithm in this case is defined as
\begin{equation}
    R_T = \sum^{T-1}_{t=0} \E_{\vx_t, \vu_t|\vx_0} [c(\vx_t, \vu_t) - A(\vpi^{*}, \vx_0)].
    \label{eq: Regret Definition}
\end{equation}

Studying the average cost criterion for general continuous state-action spaces is challenging even when the dynamics are known, since the average cost exists only for special classes of nonlinear systems~\citep{average_cost_survey}. In the following, we impose assumptions on the dynamics and policy class $\Pi$ that enable our theoretical analysis. 
\subsection{Assumptions}
Imposing continuity on $\vf^*$ is quite common in the control theory~\citep{khalil2015nonlinear} and reinforcement learning literature~\citep{curi2020efficient,sussex2022model, sukhija2024optimistic}. To this end, for our analysis, we make the following assumption.
\begin{assumption}[Continuity of $\vf^*$ and $\vpi$]
\label{ass:lipschitz_continuity}
The dynamics model $\vf^*$ and all $\vpi \in \Pi$ are continuous. 

\end{assumption}
 Next, we make an assumption on the system's stochastic disturbances.
\begin{assumption}[Process noise distribution]
\looseness=-1
The process noise is i.i.d. Gaussian with variance $\sigma^2$, i.e., $\vw_t \stackrel{\mathclap{i.i.d}}{\sim} \setN(\vzero, \sigma^2\mI)$.
\label{ass:noise_properties}
\end{assumption}
Our analysis can be extended for the more general heteroscedastic case, where $\sigma$ depends on $\vx$. However, for simplicity, we focus on the homoscedastic setting. 
In the following, we make assumptions on our policy class. To this end, we first introduce the class of $\setK_{\infty}$ functions.
\begin{definition}[$\setK_{\infty}$-functions]
The function $\xi: \R_{\geq 0} \to \R_{\geq 0}$ is of class $\setK_{\infty}$, if it is continuous, strictly increasing, $\xi(0) = 0$ and $\xi(s) \to \infty$ for $s \to \infty$.
\end{definition}

\begin{assumption}[Policies with bounded energy]
We assume there exists $\kappa, \xi \in \setK_{\infty}$,
    positive constants $K, C_u, C_l$ with $C_u > C_l$,  and $\gamma \in (0, 1)$ such that for each $\vpi \in \Pi$ we have,
\begin{itemize}[leftmargin=*]
    \item[] \label{assumption: Stability} {\em Bounded energy:}
    There exists a Lyapunov function $V^{\vpi}: \setX \to [0, \infty)$ for which  $\forall \vx, \vx' \in \setX$,
    \begin{align*} |V^{\vpi}(\vx) - V^{\vpi}(\vx')| &\leq \kappa(\norm{\vx-\vx'}) \tag{uniform continuity}\\
C_l \xi(\norm{\vx}) &\leq V^{\vpi}(\vx) \leq     C_u \xi(\norm{\vx}) \tag{positive definiteness}\\
        \E_{\vx_+|\vx, \vpi}[V^{\vpi}(\vx_+)] &\leq \gamma V^{\vpi}(\vx) + K \tag{drift condition}
    \end{align*}
    where $\vx_+ = \vf^*(\vx, \vpi(\vx)) + \vw$.
    \item[] {\em Bounded norm of cost:}
    \begin{equation*}
        \sup_{\vx \in \setX} \frac{c(\vx, \vpi(\vx))}{1 + V^{\vpi}(\vx)} < \infty 
    \end{equation*}
     \item[] {\em Boundedness of the noise with respect to $\kappa$:}
     \begin{equation*}
\E_{\vw}\left[\kappa(\norm{\vw})\right] < \infty, \
\E_{\vw}\left[\kappa^2(\norm{\vw})\right] < \infty
    \end{equation*}
\end{itemize}
\label{ass:Policy class}
\end{assumption}
The drift condition states that the energy between two timesteps can increase at most by $K$. 
In particular, the Lyapunov function $V^{\vpi}$ can be viewed as an energy function for the dynamical system, and the bounded energy condition above ensures that the system is not ``blowing up''. We do not perceive this as restrictive for real-world engineered systems.
Other works that study learning nonlinear dynamics~\citep{foster2020learning, sattar2022non, lale2021model}  in the nonepisodic setting also make stability assumptions such as global exponential stability for their analysis. In similar spirit,
we make the bounded energy assumption for our policy class. 
The drift condition on the Lyapunov function is also used to study the ergodicity of Markov chains for continuous state spaces~\citep{meyn2012markov, hairer2011yet}, which is crucial for our analysis of the infinite horizon behavior of the system. 
Moreover, for a very rich class of problems, the drift condition is satisfied. We highlight this in the corollary below.
\begin{lemma}
Assume $\vf^*$ is uniformly continuous and for all $\vpi \in \Pi$, $\vx \in \setX$, $\norm{\vpi(\vx)} \leq u_{\max}$. Further assume, there exists $\vpi_s \in \Pi$ such that we have constants $K, C_u, C_l$ with $C_u > C_l$, $\gamma \in (0, 1)$, $\kappa, \alpha \in \setK_{\infty}$ and a Lyapunov function $V : \setX \to [0, \infty)$ for which $\forall \vx, \vx' \in \setX$,
\begin{align*}
    |V(\vx) - V(\vx')| &\leq \kappa(\norm{\vx-\vx'}) \\
    C_l \xi(\norm{\vx}) &\leq V(\vx) \leq     C_u \xi(\norm{\vx}) \\
        \E_{\vx_+|\vx, \vpi_s}[V(\vx_+)] &\leq \gamma V(\vx) + K,
    \end{align*}
    where $\vx_+ = \vf^*(\vx, \vpi(\vx)) + \vw$.
    Then, $V$ also satisfies the drift condition for all $\vpi \in \Pi$, i.e., is a Lyapunov function for all policies. 
    \label{cor: bounded energy for bounded actions}
\end{lemma}
We prove this lemma in \cref{sec:proofs}. Intuitively, if the inputs are bounded, the energy inserted into the system by another policy is also bounded. Nearly all real-world systems have bounded inputs due to the physical limitations of actuators. For these systems, it suffices if only one policy in $\Pi$ satisfies the drift condition.

The boundedness assumptions for the cost and the noise in \cref{ass:Policy class} are satisfied for a rich class of cost and $\setK_{\infty}$ functions.

Under these assumptions, we can show the existence of the average cost solution.
\begin{theorem}[Existence of Average Cost Solution]
    Let \cref{ass:lipschitz_continuity} -- \ref{assumption: Stability} hold. Consider any $\vpi \in \Pi$ and let $P^{\vpi}$ denote its transition kernel, i.e., 
$P^{\vpi}(\vx, \setA) = \Pr{\vx_+ \in \setA |\vx, \vpi(\vx)}$ for $\setA \subseteq \setX$.
    Then $P^{\vpi}$ admits a unique invariant measure $\bar{P}^{\vpi}$, and there exists
    $C_2, C_3 \in (0, \infty)$, $\lambda \in (0, 1)$ such that
    \begin{itemize}[leftmargin=*]
        \item[] \emph{Average Cost}:
        \begin{equation*}
A(\vpi) = \lim_{T \to \infty} \frac{1}{T}\E_{\vpi} \left[ \sum^{T-1}_{t=0} c(\vx_t, \vu_t) \right] = \E_{\vx\sim \bar{P}^{\vpi}} \left[c(\vx, \vpi(x))\right]
        \end{equation*}
        \item[] \emph{Bias Cost}: Letting $B(\vpi, \vx_0) = \lim_{T \to \infty} \E_{\vpi} \left[ \sum^{T-1}_{t=0} c(\vx_t, \vu_t) -  A(\vpi) \right]$ denote the bias, we have
        \begin{equation*}
 |B(\vpi, \vx_0)|  \leq C_2 (1 + V^{\vpi}(\vx_0))   \frac{1}{1 - \lambda}
        \end{equation*}
        for all $\vx_0 \in \setX$.
    \end{itemize}
    \label{thm: existence of average cost problem}
\end{theorem}
\cref{thm: existence of average cost problem} is a crucial result for our analysis since it implies that the average cost is bounded and \textit{independent of the initial state} $\vx_0$. Furthermore, it also shows that the bias is bounded. 
The average cost criterion satisfies the following Bellman equation~\citep{puterman2014markov} below 
\begin{equation}
   B(\vpi, \vx) + A(\vpi) = c(\vx, \vpi(\vx)) + \E_{\vx_+}[B(\vpi, \vx_+)|\vx, \vpi]
   \label{eq: Bellman Equation Average Cost}
\end{equation}
Accordingly, the bias term plays an important role in the regret analysis (also notice its similarity to our regret term in \cref{eq: Regret Definition}).

Thus far, we have only made assumptions that make the average cost problem tractable. In the following, we make an assumption on the dynamics that allow us to learn it from data.  Moreover, we assume that at each step $n$ we learn a mean estimate $\vmu_n$ of $\vf^*$ and can quantify our uncertainty $\vsigma_n$ over the estimate.
More formally, we learn a well-calibrated statistical model of $\vf^*$ as defined below.
\begin{definition}[Well-calibrated statistical model of $\vf^*$, \cite{rothfuss2023hallucinated}]
\label{definition: well-calibrated model}
    Let $\setZ \defeq \setX \times \setU$.
    An all-time well-calibrated statistical model of the function $\vf^*$ is a sequence $\{\setM_{n}(\delta)\}_{n \ge 0}$, where
    \begin{align*}
        \setM_n(\delta) \defeq \left\{\vf: \setZ \to \R^{d_x} \mid \forall \vz \in \setZ, \forall j \in \set{1, \ldots, d_x}: \abs{\mu_{n, j}(\vz) - f_j(\vz)} \le \beta_n(\delta) \sigma_{n, j}(\vz)\right\},
    \end{align*}
    if, with probability at least $1-\delta$, we have $\vf^* \in \bigcap_{n \ge 0}\setM_n(\delta)$.
    Here, $f_{j}$, $\mu_{n, j}$ and $\sigma_{n, j}$ denote the $j$-th element in the vector-valued functions $\vf$, $\vmu_n$ and $\vsigma_n$ respectively, and $\beta_n(\delta) \in \Rzero$ is a scalar function that depends on the confidence level $\delta \in (0, 1]$ and which is monotonically increasing in $n$. 
\end{definition}
Next, we assume that $\vf^*$ resides in a Reproducing Kernel Hilbert Space (RKHS) of vector-valued functions and show that this is sufficient for us to obtain a well-calibrated model.
\begin{assumption}
We assume that the functions $f^*_j$, $j \in \set{1, \ldots, d_x}$ lie in a RKHS with kernel $k$ and have a bounded norm $B$, that is $\vf^* \in \setH^{d_x}_{k, B}$, with $\setH^{d_x}_{k, B} = \{\vf \mid \norm{f_j}_k \leq B, j=1, \dots, d_x\}$. Moreover, we assume that $k(\vz, \vz) \leq \sigma_{\max}$ for all $\vz \in \setZ$.
\label{ass:rkhs_func}
\end{assumption}
\cref{ass:rkhs_func} allows us to model $\vf^*$ with GPs for which the mean and epistemic uncertainty (${\bm \mu}_n(\vz) = [\mu_{n,j} (\vz)]_{j\leq d_x}$, and $\vsigma_n(\vz) = [\sigma_{n,j} (\vz)]_{j\leq d_x}$) have an analytical formula
\begin{equation}
\begin{aligned}
\label{eq:GPposteriors}
        \mu_{n,j} (\vz)& = {\bm{k}}_{n}^\top(\vz)({\bm K}_{n} + \sigma^2 \bm{I})^{-1}\vy_{1:n}^j 
        ,  \\
     \sigma^2_{n, j}(\vz) & =  k(\vx, \vx) - {\bm k}^\top_{n}(\vz)({\bm K}_{n}+\sigma^2 \bm{I})^{-1}{\bm k}_{n}(\vx),
\end{aligned}
\end{equation}
Here, $\vy_{1:n}^j$ corresponds to the noisy measurements of $f^*_j$, i.e., the observed next state from the transitions dataset $\setD_{1:n}$,
$\vk_n = [k(\vz, \vz_i)]_{i\leq nT}, \vz_i \in \setD_{1:n}$, and $\bm{K}_n = [k(\vz_i, \vz_l)]_{i, l\leq nT}, \vz_i, \vz_l \in \setD_{1:n}$ is the data kernel matrix. The restriction on the kernel $k(\vz, \vz) \leq \sigma_{\max}$ implies boundedness of $\vf^*$ and has also appeared in works studying the episodic setting for nonlinear systems~\citep{mania2020active, kakade2020information, curi2020efficient, sukhija2024optimistic, wagenmaker2023optimal}. We can also define $\vf^*$ such that $\vx_k = \vx_{k-1} + \vf^*(\vx_{k-1}, \vu_{k-1}) + \vw_{k-1}$ in which case the boundedness of $\vf^*$ captures many real-world systems.

\begin{lemma}[Well calibrated confidence intervals for RKHS, \citet{rothfuss2023hallucinated}]
    Let $\vf^* \in \setH_{k,B}^{d_x}$.
Suppose ${\vmu}_n$ and $\vsigma_n$ are the posterior mean and variance of a GP with kernel $k$, c.f., \Cref{eq:GPposteriors}.
There exists $\beta_n(\delta) \propto \sqrt{\Gamma_n}$, for which the tuple $(\vmu_n, \vsigma_n, \beta_n(\delta))$ is a well-calibrated statistical model of $\vf^*$.
\label{lem:rkhs_confidence_interval}
\end{lemma}
\looseness=-1
In summary, in the RKHS setting, a GP is a well-calibrated model. 
For more general models like Bayesian neural networks (BNNs), methods such as \cite{kuleshov2018accurate} can be used for calibration. Our results can also be extended beyond the RKHS setting to other classes of well-calibrated models similar to~\cite{curi2020efficient}.
\clearpage

%% file: mainmatter/02_algorithm.tex
\section{\neorl}\label{sec:algorithm}
\begin{algorithm}[t]
    \caption{\textbf{\neorl:}  \textsc{Nonepisodic Optimistic RL}}
    \begin{algorithmic}[]
        \STATE {\textbf{Init:}}{ Aleatoric uncertainty $\sigma$, Probability $\delta$, Statistical model $(\vmu_0, \vsigma_0, \beta_0(\delta))$, $H_{0}$}
        \FOR{$n=1, \ldots, N$}{
            \vspace{-0.5cm}
            \STATE {
            \begin{align*}
                &\vpi_n = \argmin_{\vpi \in \Pi} \min_{\vf \in \setM_{n-1} \cap \setM_0} A(\vpi, \vf) \quad \hspace{5em} &&\text{\ding{228} Prepare policy} \\
                  &\mathcal{D}_n \leftarrow \textsc{Rollout}(\vpi_n) \text{ for } H_n \text{ steps (\cref{equation: definition of Hn})} \quad &&\text{\ding{228} Collect measurements for horizon $H_n$ } \\
                  &\text{Update } (\vmu_n, \vsigma_n, \beta_n) \leftarrow \setD_n \quad &&\text{\ding{228} Update statistical model $\setM_n$}
                \end{align*}
                }
              }
        \ENDFOR
    \end{algorithmic}
    \label{alg:neorl}
\end{algorithm}
In the following, we present our algorithm: \textbf{N}on\textbf{e}pisodic \textbf{O}ptimistic \textbf{RL} (\neorl) for efficient nonepisodic exploration in continuous state-action spaces. \neorl builds on recent advances in episodic RL~\citep{kakade2020information, curi2020efficient, sukhija2024optimistic, treven2024ocorl} and leverages the optimism in the face of uncertainty paradigm to pick policies that are optimistic w.r.t.~the dynamics within our calibrated statistical model as follows 
\begin{align}
    (\vpi_n, \vf_n) &\defeq \argmin_{\vpi \in \Pi,\; \vf \in \setM_{n-1} \cap \setM_0} A(\vpi, \vf).      \label{eq:exploration_op_optimistic}
\end{align}
\looseness=-1
Here, $\vf_n$ is a dynamical system such that the cost by controlling $\vf_n$ with its optimal policy $\vpi_n$ is the lowest among all the plausible systems from $\setM_{n-1} \cap \setM_0$. Note, from \cref{lem:rkhs_confidence_interval} we have that $\vf^* \in \setM_{n-1} \cap \setM_0$ (with high probability) and therefore the solution to \cref{eq:exploration_op_optimistic} gives an optimistic estimate for the average cost. We take the intersection of $\setM_{n-1}$ with $\setM_0$ to ensure that we maintain at least the same confidence about our model as at the beginning, i.e., $n=0$, during learning.

\neorl proceeds in the following manner. Similar to \cite{jaksch10a}, we bin the total time $T$ the agent spends interacting in the environment into $N$ ``artificial'' episodes. At each episode, we pick a policy according to \cref{eq:exploration_op_optimistic} and roll it out for $H_n$ steps on the system. Next, we use the data collected during the rollout to update our statistical model and continue to the next episode \textit{without resetting} the system back to the initial state $\vx_0$.

\paragraph{Picking the horizon $H_n$} 
The horizon $H_n$ regulates how long we roll out the policy or how often we update our statistical model. We propose the following selection criteria for $H_n$.
\begin{align}
\label{equation: definition of Hn}
    H_n &= \max\left(\widehat{H_n}, H_0\right), \notag \\
    \widehat{H_n} &= \argmax_{H \ge 1} H+1 \notag \\
    &\text{s.t.} \sum^{H}_{k=1} \sum^{d_x}_{j=1}\log\left(1 + \sigma^{-2} \sigma^2_{n-1, j}(\vz_{k, n})\right) \le \log(2),
\end{align}
where $H_0 > 0$ is a minimal horizon we want to maintain.
The last term in \cref{equation: definition of Hn} measures the information gain~\citep{sukhija2024optimistic} the agent obtains for a rollout of length $H$ at episode $n$. Crucially, we update our model once the agent has acquired more than one-bit of information, i.e., $\sum^{H}_{k=1} \sum^{d_x}_{j=1}\log\left(1 + \sigma^{-2} \sigma^2_{n-1, j}(\vz_{k, n})\right) > \log(2)$. Furthermore, We can keep track of the information gain online during rollouts and switch the policy once we have collected sufficient data/information.
The algorithm is summarized in \cref{alg:neorl}.

\subsection{Theoretical Results}
In the following, we study the theoretical properties for \neorl and provide a first-of-its-kind bound on the cumulative regret for the average cost criterion for general nonlinear dynamical systems. Our bound depends on the {\em maximum information gain} of kernel $k$~\citep{srinivas}, defined as
\begin{equation*}
    {\Gamma}_{T}(k) = \max_{\setA \subset \setX \times \setU; |\setA| \leq T}  \frac{1}{2}\log\det{\mI + \sigma^{-2} {\bm K}_T}.
\end{equation*}
$\Gamma_T$ represents the complexity of learning $\vf^*$ from $T$ data points and
is sublinear for a very rich class of kernels (e.g., $\setO(\log^{d_{x} + d_{u} +1}(T))$ for the exponential (RBF) kernel, $\setO((d_{x} + d_{u})\log(T))$ for the linear kernel). In \cref{sec:proofs}, we report the dependence of $\Gamma_T$ on $T$ in \cref{table: gamma magnitude bounds for different kernels}. 
\begin{theorem}[Cumulative Regret of \neorl]
Let \cref{ass:lipschitz_continuity} -- \ref{ass:rkhs_func} hold, and define $H_{0}$ as the smallest integer such that
\begin{equation*}
     H_{0} > \frac{\log\left(\sfrac{C_u}{C_l}\right)}{\log\left(\sfrac{1}{\gamma}\right)}.
\end{equation*}
Then with probability at least $1-\delta$, we have the following regret for \neorl
\begin{equation*}
        R_T \leq C(\vx_0, K, \gamma)\Gamma_T\sqrt{T}.
\end{equation*}
    
with $C(\vx_0, K, \gamma)$ being bounded constant for bounded $\norm{\vx_0}$, $K$, and $\gamma < 1$.
\label{thm: Regret bound neo rl}
\end{theorem}
\looseness=-1
\cref{thm: Regret bound neo rl} gives sublinear regret for a rich class of RKHS functions. Moreover, it also gives a minimal horizon $H_{0}$ that we need to maintain before switching to the next policy. Even for the linear case, fast switching between stable controllers can destabilize the closed-loop system. We ensure this does not happen in our case by having a minimal horizon of $H_{0}$. 
\cref{thm: Regret bound neo rl} can also be derived beyond the RKHS setting for a more general class of well-calibrated models. In this case, the maximum information gain is replaced by the model complexity from~\cite{curi2020efficient} (c.f., \cite{curi2020efficient, sukhija2024optimistic} for further detail).

In the following, we give an intuitive proof sketch for \cref{thm: Regret bound neo rl}. The detailed proof is provided in \cref{sec:proofs}.

\paragraph{Proof sketch} 
\looseness=-1
The proof can be split into three main steps. First, we show the ergodicity of the closed-loop system, a sufficient condition for showing the existence of the average cost and bias term, i.e.,~\cref{thm: existence of average cost problem}, for every policy $\vpi \in \Pi$ under \cref{ass:lipschitz_continuity} -- \ref{assumption: Stability}. For this, we use elementary results on Markov chains in measurable spaces from~\citet{meyn2012markov, hairer2011yet}. Second, we show that under \cref{ass:rkhs_func}, the optimistic system selected in \cref{eq:exploration_op_optimistic}, retains the same properties as the true system $\vf^*$, e.g., stability, and therefore also is ergodic. Crucial to show this is that the true system $\vf^*$ and the optimistic system $\vf_n$ are at most $\beta_n \vsigma_n$ apart. Finally, in the third step, we show that as we update our model and policy every $H_n$ steps, our selection criteria for the horizon (\cref{equation: definition of Hn}) retains the system properties from above, and our accumulated model uncertainties across $T$ environment steps grow with the rate $\Gamma_T$. For the latter, we use the analysis from \citet{kakade2020information} for the episodic case, to bound the deviation between the optimistic average cost and the true average cost.

\subsection{Practical Modifications}\label{sec:practical modification}
For testing \neorl, we make three modifications that simplify its deployment in practice in terms of implementation and computation time.
First, instead of adaptively selecting the horizon $H_n$ we pick a fixed horizon $H$ during the experiment. This makes the planning and training of the agent easier.
Next, we use a receding horizon controller, i.e., model predictive control (MPC)~\citep{mpc}, instead of directly optimizing for the average cost in \cref{eq:exploration_op_optimistic}. MPC is widely used to obtain a feedback controller for the infinite horizon setting. Moreover, while for linear systems, the Riccati equations~\citep{anderson2007optimal} provide an analytical solution to \cref{eq:average cost formulation}, no such solution exists for the nonlinear case and MPC is commonly used as an approximation. Further, under additional assumptions on the cost and dynamics, MPC also obtains a policy with bounded average cost, which is crucial for the nonepisodic case (c.f.,~\cref{assumption: Stability}). We use the iCEM optimizer for planning~\citep{iCem}. Finally, 
instead of optimizing over $\setM_n \cap \setM_0$, we optimize directly over $\setM_n$. This allows us to use the reparameterization trick from \cite{curi2020efficient} and obtain a simple and tractable optimization problem. In summary, for each step $t$ in the environment, we solve the following optimization problem
    \begin{align}
    &\min_{\vu_{0:H_{\text{MPC}}-1}, \veta_{0;H_{\text{MPC}}-1}} \E \left[\sum_{h=0}^{H_{\text{MPC}}-1} c(\vhx_h, \vu_h) \right],        \label{eq:MPC_practical} \\ \text{ s.t. }  \vhx_{h+1} &= \vmu_{n-1}(\vhx_h, \vu_h) + \beta_{n-1}(\delta)\vsigma_{n-1}(\vhx_h, \vu_h) \veta_h + \vw_h \notag \;  \text{ and } \vhx_0 = \vx_t. \notag
\end{align}
Here $H_{\text{MPC}}$ is the MPC horizon. We take the first input from the solution of the problem above, i.e., $\vu^*_0$, and execute this in the system. We then repeat this procedure for $H$ steps and then update our statistical model $\setM_n$.
The resulting optimization above considers a larger action space as it includes the hallucinated controls $\veta$ as additional input variables. The hallucinated controls are introduced through the reparameterization trick from~\citet{curi2020efficient} and are used to directly optimize over models in $\vf \in \setM_{n}$.
Moreover, the final algorithm can be seen as a natural extension to H-UCRL~\citep{curi2020efficient} for the nonepisodic setting. We 
summarize the algorithm in \cref{sec: Experimental details} \cref{alg:neorl_pratical}. 
Note while these modifications deviate from our theoretical analysis, empirically they work well for GP and BNN models, c.f.,~\cref{sec:experiments}.

%% file: mainmatter/03_experiments.tex
\section{Experiments}
\label{sec:experiments}
\begin{figure}[ht]
    \centering
    \includegraphics[width=\textwidth]{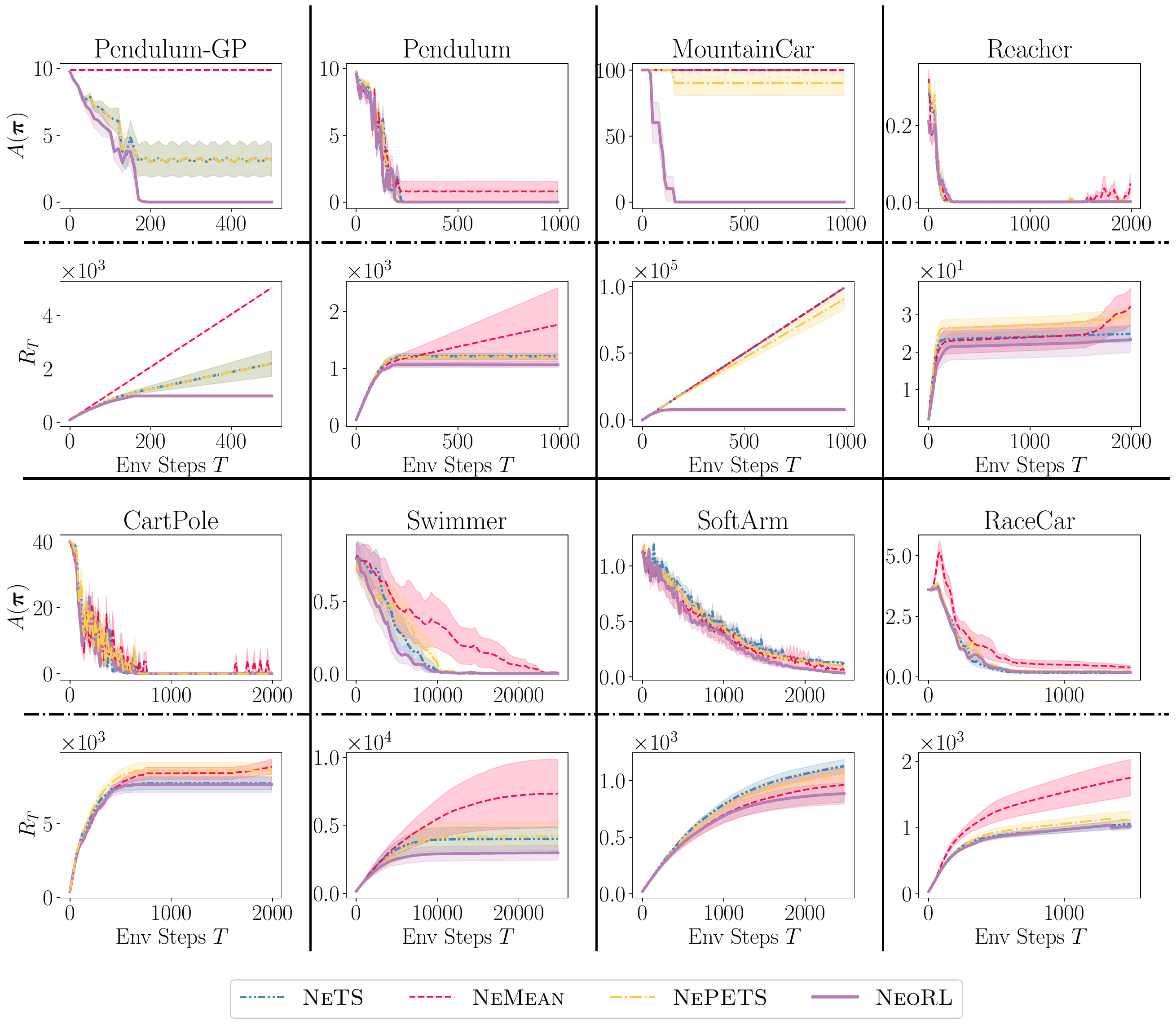}
    \caption{Average reward $A(\vpi)$ and cumulative regret $R_T$ over ten different seeds for all environments. We report the mean performance with one standard error as shaded regions. During all experiments, the environment is never reset. 
    For all baselines, we model the dynamics with probabilistic ensembles, except in the Pendulum-GP experiment, where GPs are used instead. \neorl significantly outperforms all baselines and converges to the optimal average reward, $A(\vpi^*) = 0$, showing sublinear cumulative regret $R_T$ for all environments.}
    \label{fig:average_performance_and_regret}
\end{figure}
We evaluate \neorl on the Pendulum-v1 and MountainCar environment from the OpenAI gym benchmark suite~\citep{brockman2016openai}, Cartpole, Reacher, and Swimmer from the DeepMind control suite~\citep{tassa2018deepmind}, the racecar simulator from~\cite{kabzan2020amz}, and a soft robotic arm from~\cite{arman_tekinalp_2024_10883271}. The 
swimmer and the soft robotic arm are fairly high-dimensional systems -- the
swimmer has a 28-dimensional state and 5-dimensional action space, and the
soft arm is represented by a 58-dimensional state and has a 12-dimensional action space. 
All environments are never reset during learning. Moreover, the Pendulum-v1, MountainCar, CartPole, and Reacher environments operate within a bounded domain and thus inherently satisfy \cref{assumption: Stability}. The swimmer, racecar, and soft arm can operate in an unbounded domain
but have a cost function that penalizes the distance between the system's state $\vx_t$ and a target state $\vx^*$. Therefore, the cost encourages the system to move towards the target and remain within a bounded domain. 
\paragraph{Baselines} 
\looseness=-1
In the episodic setting, resets can be used to control the exploration space for the agent. However, in the absence of resets, the agent can explore arbitrarily and end up in states that are irrelevant to the task at hand. Moreover, the agent has to follow an uninterrupted chain of experience, which makes the nonepisodic setting
the most challenging one in RL~\citep{kakade2003sample}. 
Accordingly, 
there are only a few algorithms that consider this setting (c.f.,~\cref{sec:related_work}). 
In this work, we focus on model-based RL (MBRL) algorithms due to their sample efficiency.
In particular, we adopt common MBRL methods for our setting.
MBRL algorithms typically differentiate in three ways; (\emph{i}) propagating dynamics for planning~\citep{chua2018pets, osband2017posterior, kakade2020information, curi2020efficient}, (\emph{ii}) representation of the dynamics model~\citep{ha2018recurrent, hafner2019learning, kipf2019contrastive}, and (\emph{iii}) types of planners~\citep{williams2017information, hafner2019dream, iCem}. \neorl is independent to the choice of representation or planners.
Therefore, we focus on (\emph{i}) and use probabilistic ensembles~\citep{lakshminarayanan2017simple} and GPs for modeling our dynamics and MPC with iCEM~\citep{iCem} as the planner. Common techniques to propagate the dynamics for planning are using the mean, trajectory sampling~\citep{chua2018pets}, and Thompson sampling~\citep{osband2017posterior}. We adapt these three for our setting similar to as discussed in \cref{sec:practical modification}.
For all experiments with probabilistic ensembles, we consider TS1 from~\citet{chua2018pets} for trajectory sampling, and for the GP experiment, we use
distribution sampling from \citet{chua2018pets}. We call the three baselines
\textsc{NeMean} (nonepisodic mean), \textsc{NePETS} (nonepisodic PETS), and \textsc{NeTS} (nonepisodic Thompson sampling). 
\textsc{NeMean} and \textsc{NePETS} are greedy \wrt the current estimate of the dynamics, i.e., do not explicitly encourage exploration. In our experiments, we show that being greedy does not suffice to converge to the optimal average cost, that is, obtain sublinear regret. The code for our experiments is available online.\footnote{\url{https://github.com/lasgroup/opax/tree/neorl}}
\paragraph{Convergence to the optimal average cost}
In \cref{fig:average_performance_and_regret} we report the normalized average cost and cumulative regret of \neorl, \textsc{NeMean}, \textsc{NePETS}, and \textsc{NeTS}. The normalized average cost is defined such that $A(\vpi^*) = 0$ for all environments. 
We observe that \textsc{NeMean} fails to converge to the optimal average cost for the Pendulum-v1 environment for both probabilistic ensembles and a GP model. It also fails to solve the MountainCar environment and is unstable for the Reacher and CartPole. In general, \textsc{NeMean} performs the worst among all methods. This is similar to the episodic case, where using the mean model often leads to the policy ``overfitting'' to the model inaccuracies~\citep{chua2018pets}. \textsc{NePETS} performs better than the mean, however still significantly worse than \neorl. Even in the episodic setting, PETS tends to underexplore~\citep{curi2020efficient}. 
We observe the same for the nonepisodic case, especially for the MountainCar task, which is a challenging RL environment with a sparse cost.
Here \textsc{NePETS} is also not able to achieve the optimal average cost and thus does not have sublinear cumulative regret. \textsc{NeTS} performs similarly to \textsc{NePETS} and is also not able to solve the MountainCar task.

\neorl performs the best among the baselines for all experiments and converges to the optimal average cost achieving sublinear cumulative regret using only $\sim 10^3$ environment interactions. Moreover, this observation is consistent between different dynamics models (GPs and probabilistic ensembles) and environments. Even in environments that are unbounded, i.e., Swimmer, SoftArm, and RaceCar, we observe that \neorl converges to the optimal average cost the fastest. We believe this is due to the feedback control from MPC, which has a stabilizing effect.

\paragraph{Calling reset when needed}
\looseness=-1
All the experiments in \cref{fig:average_performance_and_regret} considered the nonepisodic setting where the system was never reset during learning. A special case of our theoretical analysis is the class of policies $\Pi$ that may call for a reset / ``ask for help'' whenever they end up in an undesirable part of the state space. 
In this setting, the system is typically restricted to a compact subset of the state space $\setX$, and the policy class satisfies~\cref{assumption: Stability}. For many real-world applications, such a policy class can be derived. To simulate this experiment, we consider the CartPoleBalance task in \cref{fig:cartpole_balance_with_reset}, where the goal is to balance the pole in the upright position. A reset is triggered whenever the pole drops. We again observe that \neorl achieves the best performance, i.e., lowest cumulative regret and thus learns to solve the task the fastest. Moreover, it also requires fewer resets than \textsc{NeMean}, \textsc{NePETS}, and \textsc{NeTS}. 
\begin{figure}[ht]
    \centering
    \includegraphics[width=0.75\textwidth]{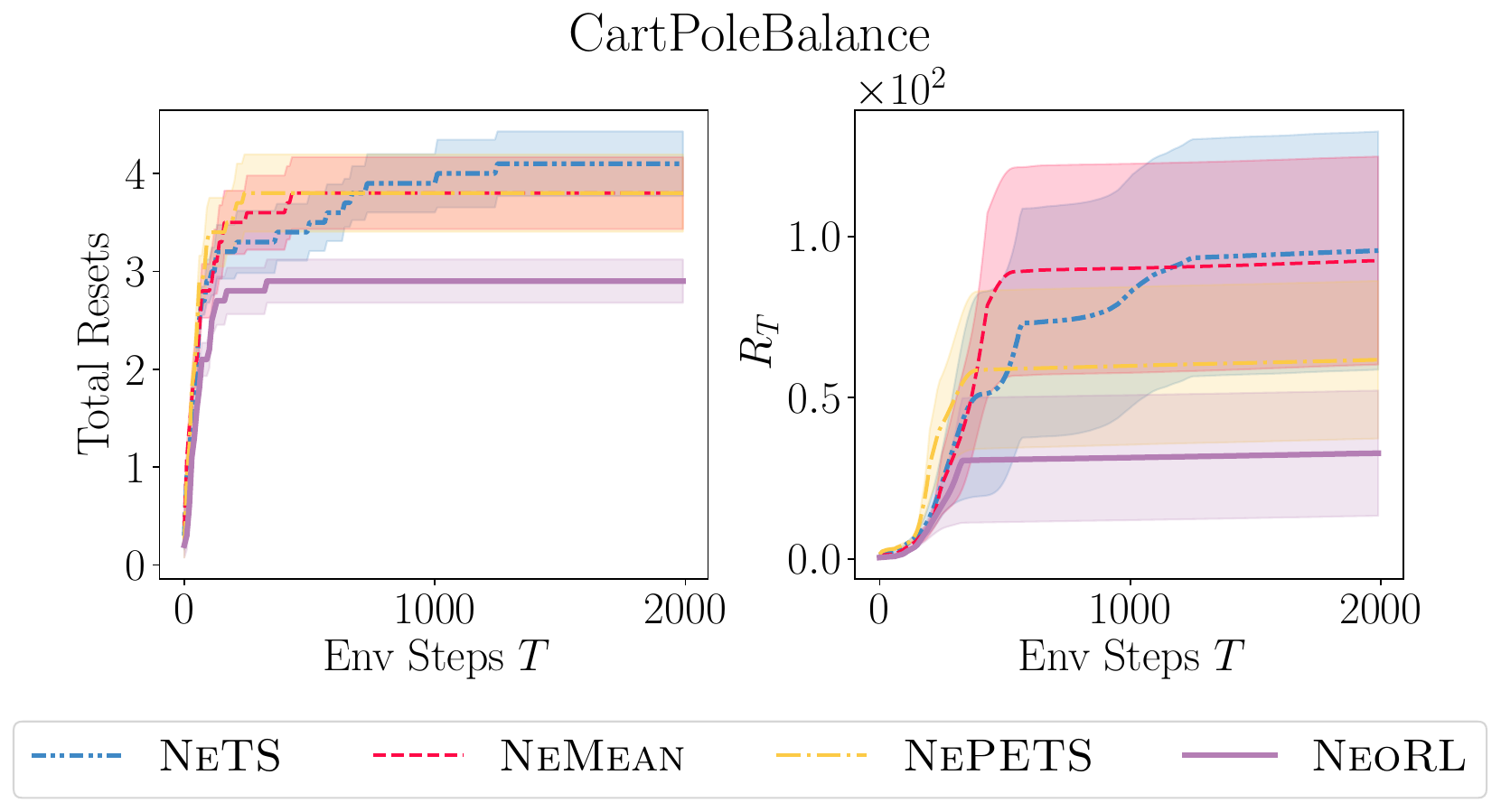}
    \caption{Total number of resets and cumulative regret $R_T$ for the cart pole balancing task over ten different seeds. We report the mean performance with one standard errors as the shaded region. The environment is automatically reset whenever the agent drops the pole. All baselines solve the task, but \neorl converges the fastest requiring fewer resets and suffering smaller regret.}
    \label{fig:cartpole_balance_with_reset}
\end{figure}

%% file: mainmatter/04_related_work.tex
\section{Related Work}\label{sec:related_work}
\looseness=-1
\paragraph{Average cost RL for finite state-action spaces}
A significant amount of work studies the average cost/reward RL setting for finite-state action spaces. Moreover, seminal algorithms such as $\text{E}^3$~\citep{kearns2002near} and $\text{R-}\max$~\citep{brafman2002r} have established PAC bounds for the nonepisodic setting. These bounds are further improved for communicating MDPs by the UCRL2~\citep{jaksch10a} algorithm, which, similar to \neorl, is based on the optimism in the face of uncertainty paradigm and picks policies that are optimistic w.r.t.~to the estimated dynamics. 
Their result is extended for weakly-communicating MDPs by REGAL~\citep{bartlett2012regal}, similar results are derived for Thompson sampling based exploration~\citep{ouyang2017learning}, and for factored-MDP~\citep{xu2020reinforcement}. 
Albeit the significant amount of work for the finite case,
progress for continuous state-action spaces has mostly been limited to linear dynamical systems.

\paragraph{Nonepisodic RL for linear systems}
There is a large body of work for nonepisodic learning with linear systems~\citep{abbasi2011regret, cohen2019learning, simchowitz2020naive, dean2020sample, lale2020logarithmic, faradonbeh2020optimism, abeille2020efficient,
treven2021learning}. For linear systems with quadratic costs, the average reward problem, also known as the linear quadratic-Gaussian (LQG), has a closed-form solution which is obtained via the Riccati equations~\citep{anderson2007optimal}.
Moreover, for LQG, stability and optimality are intertwined, making studying linear systems much easier than their nonlinear counterpart. For studying nonlinear systems, additional assumptions on their stability are usually made. 

\paragraph{Episodic RL for nonlinear systems}
\looseness=-1
In the case of nonlinear systems, guarantees have mostly been established for the episodic setting~\citep{mania2020active, kakade2020information, curi2020efficient, wagenmaker2023optimal, sukhija2024optimistic, treven2024ocorl}. In this setting, the agent begins each episode from an initial state $\vx_0$ (or initial state distribution) and
interacts with the environment for a fixed horizon $H$. It uses the data collected from the interactions to update its model. After each episode, the agent is reset back to $\vx_0$. The works mentioned above theoretically study this setting for finite-horizon MDPs and establish regret bounds for general nonlinear systems. Particularly~\citet{kakade2020information, curi2020efficient, sukhija2024optimistic, treven2024ocorl} also use an optimism-based approach similar to ours. Compared to the nonepisodic case, the analysis of episodic RL methods is simpler as resets restrict the agent's exploration around the initial state $\vx_0$ and prevent the system from blowing up or visiting states from which the agent cannot recover. 
However, as discussed in \cref{sec: introduction}, resets are often prohibitive and RL agents that learn non-episodically are preferred for many real-world applications.  

\paragraph{Nonepisodic RL beyond linear systems}
\looseness=-1
Only a few works consider the nonepisodic/single-trajectory case. For instance, a line of work studies data-driven MPC approaches focusing mostly on establishing system-theoretic guarantees such as closed-loop stability and robustness~\citep{ berberich2024overviewsystemstheoreticguaranteesdatadriven}.
From the learning side,~\cite{foster2020learning, sattar2022non} study the problem of system identification of a closed-loop globally exponentially stable dynamical system from a single trajectory. \cite{lale2021model} study the nonepisodic setting for nonlinear systems with MPC. Moreover, they consider finite-order or exponentially fading NARX systems that lie in the RKHS of infinitely smooth functions, which they further approximate with random Fourier features~\citep{rahimi2007random} $\vphi$ with feature size $D$. Further, they assume access to bounded persistently exciting inputs w.r.t.~
the feature matrix $\bm{\Phi}_t \bm{\Phi}_t^{\intercal}$. This assumption is generally tough to verify and common excitation strategies such as random exploration often don't perform well for nonlinear systems~\citep{sukhija2024optimistic}. 
The algorithm also operates in two stages, where in the first stage it performs pure exploration for system identification and in the second stage exploitation, i.e., acting greedily w.r.t.~the estimated dynamics, akin to \textsc{NeMean}. Additionally, the algorithm
requires the feature size $D$ to increase with the horizon $T$. They give a regret bound of $\setO\left(T^{\sfrac{2}{3}}\right)$ where the regret is measured w.r.t.~to the oracle MPC with access to the true dynamics. \cite{lale2021model} also assume exponential input-to-output stability of the system to avoid blow-up during exploration. Our work considers more general RKHS, naturally trades-off exploration and exploitation, does not require apriori knowledge of persistently exciting inputs and gives a regret bound of $\setO(\Gamma_T \sqrt{T })$ w.r.t.~the optimal average cost criterion. Moreover, our regret bound is similar to the ones obtained for nonlinear systems in the episodic case and Gaussian process bandits~\citep{srinivas, chowdhury2017kernelized, scarlett17a}. To the best of our knowledge, we are the first to give such a regret bound for nonlinear systems.

\paragraph{Nonepisodic Deep RL}
\looseness=-1
Standard deep RL approaches often fail in the nonepisodic setting~\citep{sharma2021autonomous}. To this end, deep RL algorithms have also been developed for the nonepisodic case. Mostly, these works focus on learning to reset and formulate it from the perspective of safety~\citep{eysenbach2017leave} (avoiding undesirable states), chaining multiple controllers~\citep{han2015learning}, skill discovery/intrinsic exploration~\citep{zhu2020ingredients, xu2020continual}, curriculum learning~\citep{sharma2021curiculum}, and learning initial state distributions from demonstrations~\citep{sharma2022state}. 
However, in contrast to us, none of the works above provide any theoretical guarantees. 

There are several extensions of model-free deep RL algorithms to the average reward setting (TRPO~\citep{atrpo}, PPO~\citep{ma2021average}, and DDPG~\citep{saxena2023off}). However, they mostly focus on maximizing the long-term behavior of the RL agent and allow for resets during learning. 
Overall, extending RL algorithms for the discounted case to the average one is still an open problem~\citep{dewanto2020average}. However, future work in this direction will benefit \neorl. Since average-reward optimizers can be used in combination with \neorl to directly minimize the average cost in a model-based policy optimization~\citep{janner2019trust} manner.




%% file: mainmatter/05_conclusion.tex
\section{Conclusion}\label{sec:limitations}
We propose, \neorl,  a novel model-based RL algorithm for the nonepisodic setting with nonlinear dynamics and continuous state and action spaces. \neorl seeks for average-cost optimal policies and leverages the model's epistemic uncertainty to perform optimistic exploration. Similar to the episodic case~\citep{kakade2020information, curi2020efficient},  we provide a regret bound for \neorl of $\setO(\Gamma_T \sqrt{T})$ for Gaussian process dynamics. To our knowledge, we are the first to obtain this result in the nonepisodic setting. We compare \neorl to other model-based RL methods on standard deep RL benchmarks. Our experiments demonstrate that \neorl, converges to the optimal average cost of $A(\vpi^*) = 0$ across all environments, suffering sublinear regret even when Bayesian neural networks are used to model the dynamics. Moreover, \neorl outperforms all our baselines across all environments requiring only $\sim 10^3$ samples for learning. 

Future work may consider deriving lower bounds on the regret of \neorl, studying different assumptions on $\vf^*$ and $\Pi$, 
and investigating different notions of optimality such as bias optimality in the nonepisodic setting~\citep{mahadevan1996average}.

%% file: backmatter/doubling_determinant_proof.tex
\section{Bounding the epistemic uncertainties with maximum information gain}
\label{section: Doubling Determinant Proof}

In this section, we prove the following lemma
\begin{lemma}
\label{lemma: sum of uncertainties bound}
    \begin{equation*}
    \sqrt{\sum^{N-1}_{n=0} \sum^{H_n-1}_{k=0} \E_{\vx^{n}_{k}, \dots \vx^{0}_{1}|  \vx_0}\left[\norm{\vsigma_n(\vx^{n}_k, \vpi_n(\vx^n_k))}^2\right]}
    \leq C'\sqrt{\Gamma_T(k)}
\end{equation*}
\end{lemma}

We update the statistical model after $H_0, H_1, \ldots$ number of environment steps. We define $H_0 = \ceil{\frac{\log\left(\sfrac{C_u}{C_l}\right)}{\log\left(\sfrac{1}{\gamma}\right)}}$, and for $n \ge 1$, we define $H_n$ as follows:
\begin{align*}
    H_n &= \max\left(\widehat{H_n}, H_0\right), \\
    \widehat{H_n} &= \argmax_{H \ge 1} H+1 \\
    &\text{s.t.} \sum^{H}_{k=1} \sum^{d_x}_{j=1}\log\left(1 + \sigma^{-2} \sigma^2_{n-1, j}(\vz_{k, n})\right) \le \log(2).
\end{align*}

For the ease of notation we denote $\vz_{k, n} = (\vx^{n}_k, \vpi_n(\vx^n_k)$). For $\vz$ we define the kernel embedding $k_{\vz} = k(\vz, \cdot)$. The covariance matrix $\mV_t: \setH \to \setH$ in the feature form is:
\begin{align}
    \mV_{t} = \mI + \frac{1}{\sigma^2}\sum_{i=1}^tk_{\vz_i}k_{\vz_i}^\top.
\end{align}
Note that we have $\vx_{t+1} = \inner{k_{\vz_t}}{\vf^*}_\setH + \vw_t$. With the design matrix $\mM_t: \setH \to \R^{t}$
\begin{align}
    \mM_t = 
    \begin{pmatrix}
     k_{\vz_1} & k_{\vz_2} & \cdots & k_{\vz_t}   
    \end{pmatrix}
\end{align}
we have $\mV_t = \mI + \frac{1}{\sigma^2}\mM_t\mM_t^\top$ and since $\mK_t = \mM_t^\top\mM_t$ we have
\begin{align}
    \determinant(\mV_t) = \determinant\left(\mI + \frac{1}{\sigma^2}\mK_{t}\right)
\end{align}
We first show that for our choice of $H_n$ the ratio between $\frac{\determinant(\mV_{n})}{\determinant(\mV_{n-1})}$ is bounded.

\begin{corollary}[Lower bound on the posterior log determinant]
\begin{equation}
   \log\left(\determinant(\mV_{n})\right) \geq \log\left(\determinant(\mV_{n-1})\right) + \log\left(1 + \sigma^{-2}\sum^{H_n}_{k=1} \norm{\vsigma_{n-1}(\vz_{k, n})}^2\right)
\end{equation}
In particular, we have
\begin{equation}
   \log\left(\frac{\determinant(\mV_{N})}{\determinant(\mV_{0})}\right) \geq \sum^N_{n=1}\log\left(1 + \sigma^{-2}\sum^{H_n}_{k=1} \norm{\vsigma_{n-1}(\vz_{k, n})}^2\right)
\end{equation}
\label{cor: recursive info gain bound}
\end{corollary}
\begin{proof}
    \begin{align*}
        &\log\left(\determinant(\mV_{n})\right) =\log\left(\determinant(\mV_{n-1})\right)  \\
        &+ \log\left(\det{\mI + \sigma^{-2}\mV^{-\sfrac{1}{2}}_{n-1}\sum^{H_n}_{k=1} \vk_{\vz_{k, n}} \vk^{\top}_{\vz_{k, n}} \mV^{-\sfrac{1}{2}}_{n-1}}\right) \\
        &\geq \log\left(\determinant(\mV_{n-1})\right) \\ &+ \log\left(1 + \tr\left(\sigma^{-2}\mV^{-\sfrac{1}{2}}_{n-1}\sum^{H_n}_{k=1} \vk_{\vz_{k, n}} \vk^{\top}_{\vz_{k, n}} \mV^{-\sfrac{1}{2}}_{n-1}\right)\right) \tag{see (*) below} \\
        &= \log\left(\determinant(\mV_{n-1})\right) +  \log\left(1 + \sigma^{-2}\sum^{H_n}_{k=1} \norm{\vk_{\vz_{k, n}}}_{\mV^{-1}_{n-1}}^2\right) \\
        &= \log\left(\determinant(\mV_{n-1})\right) +  \log\left(1 + \sigma^{-2}\sum^{H_n}_{k=1} \norm{\vsigma_{n-1}(\vz_{k, n})}^2\right)
    \end{align*}

    We prove (*) in the following, 
    first let $\vm_{k} = \sigma^{-1}\mV^{-\sfrac{1}{2}}_{n-1}\vk_{\vz_{k, n}}$, then we have 
    \begin{equation*}
        \log\left(\det{\mI + \sigma^{-2}\mV^{-\sfrac{1}{2}}_{n-1}\sum^{H_n}_{k=1} \vk_{\vz_{k, n}} \vk^{\top}_{\vz_{k, n}} \mV^{-\sfrac{1}{2}}_{n-1}}\right) = \log\left(\det{\mI + \sum^{H_n}_{k=1}\vm_t \vm^{T}_t}\right).
    \end{equation*}
    The matrix $\mM = \sum^{H_n}_{k=1}\vm_k \vm^{\top}_k$ by definition is positive semi-definite. Moreover, $\det{\mI + \mM} = \prod_{i\geq 1} (1 + \alpha_i)$, where $\alpha_i \geq 0$ are the eigenvalues of $\mM$. Furthermore, since $\alpha_i \geq 0$ and $\prod_{i\geq 1} (1 + \alpha_i) = 1 + \sum_{i\geq 1} \alpha_i + \cdots + \prod_{i\geq 1} \alpha_i$, we get $\prod_{i\geq 1} (1 + \alpha_i) \geq  1 + \sum_{i\geq 1} \alpha_i$.  Finally, since $\sum_{i\geq 1} \alpha_i = \tr(\mM)$, we get $\det{\mI + \mM} \geq 1 + \tr(\mM)$.
\end{proof}

\begin{corollary}[Upper bound on the posterior log determinant]
    \begin{equation*}
        \log\left(\determinant(\mV_{n})\right) \leq \log\left(\determinant(\mV_{n-1})\right) + \sum^{H_n}_{k=1}\sum^{d_x}_{j=1}\log\left(1 + \sigma^{-2} \sigma^2_{n-1, j}(\vz_{k, n})\right)
    \end{equation*}
    \label{cor: upper bound on the posterior log det}
\end{corollary}

\begin{proof}
    \begin{align*}
        &\log\left(\determinant(\mV_{n})\right) =\log\left(\determinant(\mV_{n-1})\right) + \log\left(\det{\mI + \mM}\right) \\
        &\leq \log\left(\determinant(\mV_{n-1})\right) + \log\left(\det{\diag\left(\mI + \mM\right)}\right) \tag{Hadamard's inequality for PSD matrices} \\
        &= \log\left(\determinant(\mV_{n-1})\right) +  \sum^{H_n}_{k=1}\sum^{d_x}_{j=1}\log\left(1 + \sigma^{-2} \sigma^2_{n-1, j}(\vz_{k, n})\right)
    \end{align*}
\end{proof}

\begin{lemma}
   Let $\tau_n = \sum_{i=1}^n H_i$, 
    then we have
    \begin{equation}
        \log\left(\frac{\determinant(\mV_{\tau_n})}{\determinant(\mV_{\tau_{n-1}})}\right) > \log(1 + \log(2))
    \end{equation}
    \label{lem: lower bound on info gain ratio neorl}
\end{lemma}
\begin{proof}
    Since $H_n = \max\{H_0, \widehat{H}_n\}$ such that 
    \begin{align*}
   \sum^{H_n}_{k=1} \sum^{d_x}_{j=1}\log\left(1 + \sigma^{-2} \sigma^2_{n-1, j}(\vz_{k, n})\right)
   &\geq \sum^{\widehat{H}_n}_{k=1} \sum^{d_x}_{j=1}\log\left(1 + \sigma^{-2} \sigma^2_{n-1, j}(\vz_{k, n})\right)\\
   &> \log(2).    
    \end{align*}
    This implies that 
    \begin{align*}
       \sum^{H_n}_{k=1}\sigma^{-2} \norm{\vsigma_{n-1, j}(\vz_{k, n})}^2 &= \sum^{H_n}_{k=1} \sum^{d_x}_{j=1}\sigma^{-2} \sigma^2_{n-1, j}(\vz_{k, n}) \\
      &\geq \sum^{H_n}_{k=1} \sum^{d_x}_{j=1}\log\left(1 + \sigma^{-2} \sigma^2_{n-1, j}(\vz_{k, n})\right) \tag{$x \geq \log(1 + x)$}
      \\&> \log(2).
    \end{align*}
    Therefore, 
    \begin{align*}
           \log\left(\determinant(\mV_{\tau_n})\right) &\geq \log\left(\determinant(\mV_{\tau_{n-1}})\right) + \log\left(1 + \sigma^{-2}\sum^{H_n}_{k=1} \norm{\vsigma_{n-1}(\vz_{k, n})}^2\right) \\
           &> \log\left(\determinant(\mV_{\tau_{n-1}})\right) + \log(1 + \log(2))
    \end{align*}
\end{proof}

Next, we prove two lemmas that will help us upper bound the sum of variances.
\begin{lemma}[Adapted Lemma 11 from \citet{abbasi2011regret}]
\label{lemma: max gain matrices}
Let $t, k \in \mathbb{N}$. Then we have:
\begin{align*}
    \sup_{\mX \ne 0}\frac{\norm{\mX^\top \mV_{t+k}\mX}}{\norm{\mX^\top \mV_{t}\mX}} \le \frac{\determinant{\mV_{t+k}}}{\determinant{\mV_t}}
\end{align*}
\end{lemma}

\begin{proof}
    Let us start with the case when $k=1$ and $\mV_{t+1} = \mV_{t} + k_{\vz_{t+1}}k_{\vz_{t+1}}^\top$.
    First notice:
    \begin{align*}
    \determinant(\mV_{t+1}) &= \determinant\left( \mV_{t} + k_{\vz_{t+1}}k_{\vz_{t+1}}^\top\right) \\
    &= \determinant(\mV_t)\determinant\left(\mI + \mV_t^{-\frac{1}{2}}k_{\vz_{t+1}}\left(\mV_t^{-\frac{1}{2}}k_{\vz_{t+1}}\right)^\top\right) \\
    &= \determinant(\mV_t)\left(1 + \norm{k_{\vz_{t+1}}^\top \mV_t^{-\frac{1}{2}}}_\setH^2\right)
\end{align*}
    
    Further we have:
    \begin{align*}
        \norm{\mX^\top k_{\vz_{t+1}}k_{\vz_{t+1}}^\top\mX}_\setH &= \norm{k_{\vz_{t+1}}^\top\mX}^2_\setH \\
        &= \norm{k_{\vz_{t+1}}^\top \mV_t^{-\frac{1}{2}} \mV_t^{\frac{1}{2}}\mX}^2_\setH  \\
        &\le \norm{k_{\vz_{t+1}}^\top \mV_t^{-\frac{1}{2}}}_\setH^2\norm{\mV_t^{\frac{1}{2}}\mX}^2_\setH \tag{Submultiplicativity}
    \end{align*}
    Hence:
    \begin{align*}
        \norm{\mX^\top \mV_{t+1} \mX}_\setH &= \norm{\mX^\top \left(\mV_{t} + k_{\vz_{t+1}}k_{\vz_{t+1}}^\top\right)\mX}_\setH \\
        &\le \norm{\mX^\top \mV_{t} \mX}_\setH + \norm{\mX^\top k_{\vz_{t+1}}k_{\vz_{t+1}}^\top\mX}_\setH \tag{Triangle inequality} \\
        &\le \norm{\mX^\top \mV_{t} \mX}_\setH + \norm{k_{\vz_{t+1}}^\top \mV_t^{-\frac{1}{2}}}_\setH^2\norm{\mV_t^{\frac{1}{2}}\mX}^2_\setH \tag{Previous step} \\
        &= \left(1 + \norm{k_{\vz_{t+1}}^\top \mV_t^{-\frac{1}{2}}}_\setH^2\right)\norm{\mX^\top \mV_{t} \mX}_\setH \\
        &= \frac{\determinant(\mV_{t+1})}{\determinant(\mV_{t})}\norm{\mX^\top \mV_{t} \mX}_\setH
    \end{align*}
Now observe:
\begin{align*}
    \frac{\norm{\mX^\top \mV_{t+k}\mX}}{\norm{\mX^\top \mV_{t}\mX}} &= \frac{\norm{\mX^\top \mV_{t+k}\mX}}{\norm{\mX^\top \mV_{t+k-1}\mX}}\frac{\norm{\mX^\top \mV_{t+k-1}\mX}}{\norm{\mX^\top \mV_{t+k-2}\mX}}\cdots\frac{\norm{\mX^\top \mV_{t+1}\mX}}{\norm{\mX^\top \mV_{t}\mX}} \\
    &\le \frac{\determinant(\mV_{t+k})}{\determinant(\mV_{t + k-1})}\frac{\determinant(\mV_{t+k-1})}{\determinant(\mV_{t+k-2})}\cdots \frac{\determinant(\mV_{t+1})}{\determinant(\mV_{t})} \\
    &= \frac{\determinant{\mV_{t+k}}}{\determinant{\mV_t}}
\end{align*}
\end{proof}

\begin{lemma}
\label{lemma: info gain lemma}
   Let $\tau_n = \sum_{i=1}^n H_i$, we have for all $h > 0$:
    \begin{align*}
        \norm{\mV_{\tau_{n-1}}^{-\frac{1}{2}}\mV_{\tau_{n-1}+h-1}^{\frac{1}{2}}}^2_{\setH} \le \bar{C}_I,
    \end{align*}
    where $\bar{C}_I=\max\left(\log(2), (1 + \sigma^{-2} \sigma_{\max})^{d_x H_0}\right)$.
\end{lemma}

\begin{proof}
    We apply \Cref{lemma: max gain matrices} with $t = \tau_{n-1}, k = h-1$ and $\mX = \mV_{\tau_{n-1}}^{-\frac{1}{2}}$
    We have:
    \begin{align*}
        \norm{\mV_{\tau_{n-1}}^{-\frac{1}{2}}\mV_{\tau_{n-1}+h-1}^{\frac{1}{2}}}^2_{\setH} &=  \norm{\mV_{\tau_{n-1}}^{-\frac{1}{2}}\mV_{\tau_{n-1}+h-1}^{\frac{1}{2}}\left(\mV_{\tau_{n-1}}^{-\frac{1}{2}}\mV_{\tau_{n-1}+h-1}^{\frac{1}{2}}\right)^\top}_{\setH} \\
        &= \norm{\mV_{\tau_{n-1}}^{-\frac{1}{2}}\mV_{\tau_{n-1}+h-1}\mV_{\tau_{n-1}}^{-\frac{1}{2}}}^2_{\setH} \\
        &= \frac{\norm{\mV_{\tau_{n-1}}^{-\frac{1}{2}}\mV_{\tau_{n-1}+h-1}\mV_{\tau_{n-1}}^{-\frac{1}{2}}}^2_{\setH}}{\norm{\mI}_\setH} \\
        &= \frac{\norm{\mV_{\tau_{n-1}}^{-\frac{1}{2}}\mV_{\tau_{n-1}+h-1}\mV_{\tau_{n-1}}^{-\frac{1}{2}}}^2_{\setH}}{\norm{\mV_{\tau_{n-1}}^{-\frac{1}{2}}\mV_{\tau_{n-1}}\mV_{\tau_{n-1}}^{-\frac{1}{2}}}_\setH} \\
        &\le \frac{\determinant(\mV_{\tau_{n-1}+h-1})}{\determinant(\mV_{\tau_{n-1}})} \tag{\Cref{lemma: max gain matrices}} \\
        &= \frac{\determinant\left(\mI + \frac{1}{\sigma^2}\mK_{\tau_{n-1}+h-1}\right)}{\determinant\left(\mI + \frac{1}{\sigma^2}\mK_{\tau_{n-1}}\right)} \\
        &\le \max\left(\log(2), (1 + \sigma^{-2} \sigma_{\max})^{d_x H_0}\right) := \bar{C}_I \tag{Definition of $H_n$}
    \end{align*}
    where in the last inequality we used \cref{cor: upper bound on the posterior log det}.
    Moreover, for the case where $H_n = \widehat{H}_n$, we have $h < \widehat{H}_n$ and therefore
    \begin{align*}
    \log\left(\determinant(\mV_{\tau_{n-1}+h-1})\right) &\leq
    \log\left(\determinant(\mV_{\tau_{n-1}})\right) + \sum^{h}_{k=1}\sum^{d_x}_{j=1}\log\left(1 + \sigma^{-2} \sigma^2_{n-1, j}(\vz_{k, n})\right)
    \\ &\leq \log\left(\determinant(\mV_{\tau_{n-1}})\right) + \sum^{\widehat{H}_n}_{k=1}\sum^{d_x}_{j=1}\log\left(1 + \sigma^{-2} \sigma^2_{n-1, j}(\vz_{k, n})\right) \\
    &\leq \log\left(\determinant(\mV_{\tau_{n-1}})\right) + \log(2).
    \end{align*}
    Whereas for the case where $H_n = H_0$, $h < H_0$ and accordingly
     \begin{align*}
    \log\left(\determinant(\mV_{\tau_{n-1}+h-1})\right) &\leq
    \log\left(\determinant(\mV_{\tau_{n-1}})\right) + \sum^{h}_{k=1}\sum^{d_x}_{j=1}\log\left(1 + \sigma^{-2} \sigma^2_{n-1, j}(\vz_{k, n})\right) \\
    &\leq \log\left(\determinant(\mV_{\tau_{n-1}})\right) + \sum^{H_0}_{k=1}\sum^{d_x}_{j=1}\log\left(1 + \sigma^{-2} \sigma^2_{n-1, j}(\vz_{k, n})\right) \\
    &\leq \log\left(\determinant(\mV_{\tau_{n-1}})\right) + d_x H_0 \log(1 + \sigma^{-2} \sigma_{\max})
    \end{align*}
\end{proof}

\begin{proof}[Proof of \Cref{lemma: sum of uncertainties bound}]
    We have:
    \begin{align*}
    &\sum^{N-1}_{n=0} \sum^{H_n-1}_{k=0} \E_{\vx^{n}_{k}, \dots \vx^{0}_{1}|  \vx_0}\left[\norm{\vsigma_n(\vx^{n}_k, \vpi_n(\vx^n_k))}^2\right] = \\
    &= \E_{\vx_T, \dots \vx_{1}|  \vx_0}\left[\sum^{N-1}_{n=0} \sum^{H_n-1}_{k=0}\norm{\vsigma_n(\vz_{k, n})}^2\right] \\
    &= \E_{\vx_T, \dots \vx_{1}|  \vx_0}\left[\sum_{j=1}^{d_{\vx}}\sum^{N-1}_{n=0} \sum^{H_n-1}_{k=0}\sigma_{n, j}(\vz_{k, n})^2\right] \\
    &= d_{\vx} \E_{\vx_T, \dots \vx_{1}|  \vx_0}\left[\sum^{N}_{n=1} \sum^{H_n}_{k=1}k_{\vz_{k, n}}^\top \mV_{\tau_{n-1}}^{-1}k_{\vz_{k, n}}\right] \tag{Same kernel in every dimension}\\
    &= d_{\vx}\E_{\vx_T, \dots \vx_{1}|  \vx_0}\left[\sum_{n=1}^N\sum_{k=1}^{H_n}\norm{\mV_{\tau_{n-1}}^{-\frac{1}{2}}k_{\vz_{k, n}}}^2_{\setH}\right]\\
    &= d_{\vx}\E_{\vx_T, \dots \vx_{1}|  \vx_0}\left[\sum_{n=1}^N\sum_{k=1}^{H_n}\norm{\mV_{\tau_{n-1}}^{-\frac{1}{2}}\mV_{\tau_{n-1}+k-1}^{\frac{1}{2}}\mV_{\tau_{n-1}+k-1}^{-\frac{1}{2}}k_{\vz_{k, n}}}^2_{\setH} \right]\\
    &\le d_{\vx}\E_{\vx_T, \dots \vx_{1}|  \vx_0}\left[ \sum_{n=1}^N\sum_{k=1}^{H_n}\norm{\mV_{\tau_{n-1}}^{-\frac{1}{2}}\mV_{\tau_{n-1}+k-1}^{\frac{1}{2}}}^2_{\setH}\norm{\mV_{\tau_{n-1}+k-1}^{-\frac{1}{2}}k_{\vz_{k, n}}}^2_{\setH}\right] \tag{Submultiplicativity} \\
    &\le \bar{C}_Id_{\vx}\E_{\vx_T, \dots \vx_{1}|  \vx_0}\left[\sum_{n=1}^N\sum_{k=1}^{H_n}\norm{\mV_{\tau_{n-1}+k-1}^{-\frac{1}{2}}k_{\vz_{k, n}}}^2_{\setH} \right]\tag{\cref{lemma: info gain lemma}}\\
    &= \bar{C}_I \E_{\vx_T, \dots \vx_{1}|  \vx_0}\left[\sum_{n=1}^N\sum_{k=1}^{H_n}\norm{\vsigma_{\tau_{n-1}+k-1}(\vz_{k, n})}^2\right] \\
    &= \bar{C}_I d_{\vx}\E_{\vx_T, \dots \vx_{1}|  \vx_0}\left[\sum_{t=1}^T\norm{\vsigma_{t-1}(\vz_{t})}^2\right] \\
    &\le \bar{C}_I \max_{\setD_{1:T}}\sum_{t=1}^T\norm{\vsigma_{t-1}(\vz_{t})}^2 \tag{Maximum is greater than expecation}\\
    &\le \frac{\bar{C}_Id_{\vx}\sigma_{\max}}{\log(1 + \sigma^{-2}\sigma_{\max})} \max_{\setD_{1:T}}\sum_{t=1}^T\log\left(1 + \frac{\sigma_{t-1}^2(\vz_t)}{\sigma^2}\right) \tag{Lemma 15 of \citet{curi2020efficient}}\\
    &= \frac{\bar{C}_Id_{\vx}\sigma_{\max}}{\log(1 + \sigma^{-2}\sigma_{\max})}\max_{\setD_{1:T}}I(\vy_{\setD_{1:T}}, f_1) \\
    &= \frac{\bar{C}_Id_{\vx}\sigma_{\max}}{\log(1 + \sigma^{-2}\sigma_{\max})} \Gamma_T(k)
    \end{align*}
    Therefore it follows:
    \begin{align*}
            \sqrt{\sum^{N-1}_{n=0} \sum^{H_n-1}_{k=0} \E_{\vx^{n}_{k}, \dots \vx^{0}_{1}|  \vx_0}\left[\norm{\vsigma_n(\vx^{n}_k, \vpi_n(\vx^n_k))}^2\right]}
    \le \sqrt{\frac{\bar{C}_Id_{\vx}\sigma_{\max}}{\log(1 + \sigma^{-2}\sigma_{\max})}} \sqrt{\Gamma_T(k)}
    \end{align*}
\end{proof}

Finally, we show that the number of episodes $N$ is bounded by $\Gamma_T$.

\begin{lemma}
\label{lemma: number of statistical model computations}
    We have $N \le \frac{1}{\bar{K}_I} \Gamma_T(k)$, where $\bar{K}_I = \log(1 + \log(2))$.
\end{lemma}

\begin{proof}
From \cref{lem: lower bound on info gain ratio neorl}, we have $ \determinant(\mV_{\tau_n}) > \exp(\bar{K}_I) \determinant(\mV_{\tau_{n-1}})$. 
     From the definition, it follows that $T = \tau_N$. 
\begin{align*}
 \determinant(\mV_{\tau_n}) >\exp(\bar{K}_I) \determinant(\mV_{\tau_{n-1}}) > \exp(\bar{K}_I)^2 \determinant(\mV_{\tau_{n-2}}) > \cdots > \exp(\bar{K}_I)^N \determinant(\mV_{0}),
\end{align*}
we have that $\exp(\bar{K}_I)^N < \frac{\determinant(\mV_{T})}{\determinant(\mV_0)}$, and hence 
\begin{align*}
N < \frac{1}{\bar{K}_I}\log\left(\frac{\determinant(\mV_{T})}{\determinant(\mV_0)}\right) \le \frac{1}{\bar{K}_I}\max_{\setD_{1:T}}\log\left(\frac{\determinant(\mV_{T})}{\determinant(\mV_0)}\right) = \frac{1}{\bar{K}_I}\Gamma_T(k).    
\end{align*} 
\end{proof}

%% file: backmatter/A_proofs.tex
\section{Proofs}\label{sec:proofs}

In this section, we prove \cref{thm: existence of average cost problem} and \cref{thm: Regret bound neo rl}. First, we start with the proof of \cref{cor: bounded energy for bounded actions}.

\begin{proof}[Proof of \cref{cor: bounded energy for bounded actions}]
We first analyze the following term $\E_{\vw}[V(\vf^*(\vx, \vpi(\vx)) + \vw) - V(\vf^*(\vx, \vpi_s(\vx)) + \vw)]$ for any $\vpi \in \Pi$.
\begin{align*}
    \E_{\vw}&[V(\vf^*(\vx, \vpi(\vx)) + \vw) - V(\vf^*(\vx, \vpi_s(\vx)) + \vw)] \\ 
    &\leq \E_{\vw}[\kappa(\norm{\vf^*(\vx, \vpi(\vx)) + \vw - (\vf^*(\vx, \vpi_s(\vx)) + \vw)})] \tag{Uniform continuity of $V$} \\
    &= \kappa(\norm{\vf^*(\vx, \vpi(\vx)) - \vf^*(\vx, \vpi_s(\vx))}) \\
    &\leq \kappa(\kappa_{\vf^*}(\norm{\vpi(\vx) - \vpi_s(\vx)})) \tag{Uniform continuity of $\vf^*$} \\
    &\leq \kappa(\kappa_{\vf^*}(2u_{\max})) \tag{Bounded inputs}.
\end{align*}
Therefore, 
\begin{align*}
    \E_{\vx'|\vpi, \vx}[V(\vx')] &= \E_{\vw}[V(\vf^*(\vx, \vpi(\vx)) + \vw)] \\
    &\leq \E_{\vw}[V(\vf^*(\vx, \vpi_s(\vx)) + \vw)] + \kappa(\kappa_{\vf^*}(2u_{\max})) \\
    &= \E_{\vx'|\vpi_s, \vx}[V(\vx')] + \kappa(\kappa_{\vf^*}(2u_{\max})) \\
    &\leq \gamma V(\vx) + K + \kappa(\kappa_{\vf^*}(2u_{\max})) \\
    &= \gamma V(\vx) + \tilde{K} \tag{$\tilde{K} = K + \kappa(\kappa_{\vf^*}(2u_{\max}))$}
\end{align*}
Hence, $V$ satisfies the drift condition for $\vpi$. Furthermore, since $V$ also satisfies positive definiteness by assumption, the bounded energy condition holds for all $\vpi \in \Pi$.
\end{proof}
\subsection{Proof of \cref{thm: existence of average cost problem}}
For proving \cref{thm: existence of average cost problem}, we invoke the results from \cite[Theorem 1.2 -- 1.3]{hairer2011yet}. For this we require that the Markov chain induced by a policy $\vpi$ satisfies the drift condition. In our setting, this corresponds to \cref{assumption: Stability}. Next, we show that the chain satisfies the following minorisation condition.
\begin{lemma}[Minorisation condition]
\label{lem: minorisation condition}
Consider the system in \cref{eq:dynamics} and let \cref{ass:lipschitz_continuity} -- \ref{assumption: Stability} hold. Let $P^{\vpi}$ denote the transition kernel for the policy $\vpi \in \Pi$, i.e., 
$P^{\vpi}(\vx, \setA) = \Pr{\vx_+ \in \setA |\vx, \vpi(\vx)}$
. Then, for all $\vpi \in \Pi$,  exists a constant $\alpha \in (0, 1)$ and a probability measure $\zeta(\cdot)$ s.t., 
\begin{equation}
    \inf_{\vx \in \setC} P^{\vpi}(\vx, \cdot) \geq \alpha \zeta(\cdot) \label{eq:minorisation condition}
\end{equation}
with $\setC \defeq \{\vx \in \setX; V^{\vpi}(\vx) \leq R\}$ for some $R >  \nicefrac{2 K}{1 - \gamma}$
\end{lemma}

\begin{proof}
    We prove it in 3 steps. First, we show that $\setC$ is contained in a compact domain. From the \Cref{assumption: Stability} we pick the function $\xi \in \setK_{\infty}$. Since $C_l\xi(0) = 0, \lim_{s \to \infty} \xi(s) = +\infty$ and $C_l \xi$ is continuous, there exists $M$ such that $C_l \xi(M) = R$. Then for $\norm{\vx} > M$ we have:
    \begin{align*}
        V^{\vpi}(\vx) \ge C_l\xi(\norm{\vx}) > \xi(M) = R.
    \end{align*}
    Therefore we have: $\setC \subseteq \setB(\vzero, M) \defeq \{\vx \mid \norm{\vx - \vzero} \le M\}$. 
    In the second step we show that $\vf(\setC, \vpi(\setC))$ is bounded, in particular we show that there exists $B > 0$ such that: $\vf(\setC, \vpi(\setC)) \subseteq \setB(\vzero, B)$. This is true since continuous image of compact set is compact and the observation:
    \begin{align*}
        \setC \subseteq \setB(\vzero, M) \implies \vf(\setC, \vpi(\setC)) \subseteq \vf(\setB(\vzero, M), \vpi(\setB(\vzero, M))).
    \end{align*}
    Since $\vf(\setB(\vzero, M), \vpi(\setB(\vzero, M)))$ is compact there exists $B$ such that $\vf(\setC, \vpi(\setC)) \subseteq \setB(\vzero, B)$.
    In the last step we prove that $\alpha \defeq 2^{-d_{\vx}}e^{-B^2/\sigma^2}$ and $\zeta$ with law of $\setN\left(0, \frac{\sigma^2}{2}\right)$ satisfy condition of \Cref{lem: minorisation condition}.
    It is enough to show that $\forall \vmu \in \setB(\vzero, B), \forall \vx \in \R^{d_\vx}$ we have:
    \begin{align*}
        \alpha \frac{1}{(2\pi)^{\frac{d_x}{2}}\left(\frac{\sigma^2}{2}\right)^{\frac{d_\vx}{2}}}e^{-\frac{\norm{\vx}^2}{\sigma^2}} \le \frac{1}{(2\pi)^{\frac{d_x}{2}}(\sigma^2)^{\frac{d_\vx}{2}}}e^{-\frac{\norm{\vx - \vmu}^2}{2\sigma^2}}
    \end{align*}
    which can be proven with simple algebraic manipulations.
\end{proof}
Through the minorisation condition and \cref{assumption: Stability}, we can prove the ergodicity of the closed-loop system for a given policy $\vpi \in \Pi$.
\begin{theorem}[Ergodicity of closed-loop system]
Let \cref{ass:lipschitz_continuity} -- \ref{assumption: Stability}, consider any probability measures $\zeta_1$, $\zeta_2$, and $\theta > 0$, define $P^{\vpi} \zeta$, $\norm{\varphi}_{1 + \theta V^{\vpi}}$,
$\rho^{\vpi}_{\theta}$ as
\begin{align*}
    \left(P^{\vpi} \zeta\right) (\setA) &= \int_{\setX} P^{\vpi}(\vx, \setA) \zeta(d\vx) \\
     \norm{\varphi}_{1 + \theta V^{\vpi}} &= \sup_{\vx \in \setX} \frac{|\varphi(\vx)|}{1 + \theta V^{\vpi}(\vx)} \\
    \rho^{\vpi}_{\theta}(\zeta_1, \zeta_2) &= \sup_{\varphi: \norm{\varphi}_{1 + \theta V^{\vpi}} \leq 1} \int_{\setX} \varphi(\vx) (\zeta_1 - \zeta_2)(d\vx) = \int_{\setX} (1 + \theta V^{\vpi}(\vx))|\zeta_1 - \zeta_2|(d\vx).
\end{align*}
We have for all $\vpi \in \Pi$, that $P^{\vpi}$ admits a unique invariant measure $\bar{P}^{\vpi}$. 
Furthermore, there exist constants $C_1 >0$, $\theta > 0$, $\lambda \in (0,1)$ such that
\begin{align}
    \rho^{\vpi}_{\theta}(P^{\vpi}\zeta_1, P^{\vpi}\zeta_2)
    &\leq \lambda \rho^{\vpi}_{\theta}(\zeta_1, \zeta_2)\tag{1} \\
    \norm{\E_{\vx \sim (P^{\vpi})^t }\left[\varphi(\vx)\right] -\E_{\vx\sim \bar{P}^{\vpi}} \left[\varphi(\vx)\right]}_{1 + V^{\vpi}} &\leq C_1 \lambda^{t}\norm{\varphi-\E_{\vx\sim \bar{P}^{\vpi}} \left[\varphi(\vx)\right]}_{1 + V^{\vpi}} \tag{2}.
\end{align}

holds for every measurable function $\varphi: \setX \rightarrow \setR$ with $\norm{\varphi}_{1 + V^{\vpi}}<\infty$. Here $(P^{\vpi})^t$ denotes the $t$-step transition kernel under the policy $\vpi$. 

Moreover, $\theta = \nicefrac{\alpha_0}{K}$, and
\begin{equation}
    \lambda = \max \left\{1 - (\alpha - \alpha_0), \frac{2 +  \nicefrac{R}{K}\alpha_0 \gamma_0}{2 +  \nicefrac{R}{K}\alpha_0}\right\}
\end{equation}
for any $\alpha_0 \in (0, \alpha)$ and $\gamma_0 \in (\gamma + 2\nicefrac{K}{R} ,1)$.

\label{thm:geometric_ergodicity}
\end{theorem}
\begin{proof}
From \cref{ass:Policy class}, we have a value function for each policy that satisfies the drift condition. Furthermore, in \cref{lem: minorisation condition} we show that our system also satisfies the minorisation condition for all policies. Under these conditions, we can use the results from~\citet[Theorem 1.2. -- 1.3.]{hairer2011yet}. 
\end{proof}
Note that $\norm{\cdot}_{1 + \theta V^{\vpi}}$ represents a family of equivalent norms for any $\theta > 0$. Now we prove \cref{thm: existence of average cost problem}.
\begin{proof}[Proof of \cref{thm: existence of average cost problem}]
From~\cref{thm:geometric_ergodicity}, we have
\begin{equation*}
    \rho^{\vpi}_{\theta}((P^{\vpi})^{t+1}, (P^{\vpi})^{t}) = \rho^{\vpi}_{\theta}(P^{\vpi} (P^{\vpi})^t, P^{\vpi}(P^{\vpi})^{t-1})
    \leq \lambda^{t} \rho^{\vpi}_{\theta}(P^{\vpi}\delta_{\vx_0}, \delta_{\vx_0}),
\end{equation*}
where $\delta_{\vx_0}$ is the dirac measure. Therefore, $(P^{\vpi})^{t}$ is a Cauchy sequence. Furthermore, $\rho^{\vpi}_{\theta}$ is complete for the set of probability measures integrating $V$, thus
$\rho^{\vpi}_{\theta}((P^{\vpi})^{t}, \bar{P}^{\vpi}) \to 0$ for $t \to \infty$ (c.f.,~\cite{hairer2011yet} for more details). In particular, we have for $\varphi$ such that $\norm{\varphi}_{1 + \theta V^{\vpi}} \leq 1$, 
\begin{equation*}
    \lim_{t \to \infty} \int_{\setX} \varphi(\vx) (P^{\vpi})^{t}(d\vx) = \int_{\setX} \varphi(\vx)  \bar{P}^{\vpi}(d\vx). 
\end{equation*}
Note that since all $\norm{\cdot}_{1 + \theta V^{\vpi}}$ norms are equivalent for $\theta > 0$, if $\norm{c}_{1 + V^{\vpi}} \leq C$ (\cref{assumption: Stability}), then 
$\norm{c}_{1 + \theta V^{\vpi}} \leq C'$ for some $C' \in (0, \infty)$. Furthermore, note that $ c(\cdot) \geq 0$. 
 Therefore, 
 \begin{align*}
     \int_{\setX} c(\vx)  \bar{P}^{\vpi}(d\vx)
     &= \lim_{t\to\infty} \int_{\setX} c(\vx)  (P^{\vpi})^{t}(d\vx) \\
     &\leq C \lim_{t\to\infty} \int_{\setX} (1 + V^{\vpi}(\vx))  (P^{\vpi})^{t}(d\vx) \\
     &= C + C \lim_{t\to\infty}  \E_{\vx \sim (P^{\vpi})^{t}}[V^{\vpi}(\vx)] \\
     &= C + C \lim_{t\to\infty}  \E_{\vx \sim (P^{\vpi})^{t-1}}[\E_{\vx' \sim (P^{\vpi})}[V^{\vpi}(\vx')|\vx]] \\
     &\leq C + C \left(\lim_{t\to\infty} \gamma \E_{\vx \sim (P^{\vpi})^{t-1}}[V^{\vpi}(\vx)] + K\right) \tag{\cref{assumption: Stability}}\\
     &\leq C + C \lim_{t\to\infty} \gamma^t V^{\vpi}(\vx_0) + K \frac{1-\gamma^t}{1-\gamma} \\
     &= C  \left(1 + K \frac{1}{1-\gamma}\right)
 \end{align*}
 In summary, we have $\E_{\vx\sim \bar{P}^{\vpi}} \left[c(\vx)\right] \leq C \left(1 + K \frac{1}{1-\gamma}\right)$

Consider any $t > 0$,  and note that from \cref{thm:geometric_ergodicity} we have 
\begin{align*}
    \norm{\E_{\vx \sim (P^{\vpi})^t }\left[c(\vx)\right] -\E_{\vx\sim \bar{P}^{\vpi}} \left[c(\vx)\right]}_{1 + V^{\vpi}} &= \sup_{\vx_0 \in \setX} \frac{| \E_{\vx \sim (P^{\vpi})^t }\left[c(\vx)\right] -\E_{\vx\sim \bar{P}^{\vpi}} \left[c(\vx)\right]|}{1 + V^{\vpi}(\vx_0)} \\
    &\leq C_1 \lambda^{t}\norm{c-\E_{\vx\sim \bar{P}^{\vpi}} \left[c(\vx)\right]}_{1 + V^{\vpi}} \tag{\cref{thm:geometric_ergodicity}}\\
    &\leq C_1 \lambda^{t}\norm{c}_{1 + V^{\vpi}} + C_1 \lambda^{t}\E_{\vx\sim \bar{P}^{\vpi}} \left[c(\vx)\right] \\
    &= C_2 \lambda^{t},
\end{align*}
where $C_2 = C_1 (\norm{c}_{1 + V^{\vpi}} + C K \frac{1}{1-\gamma})$. 

Moreover, since the inequality holds for all $\vx_0$, we have 
\begin{equation*}
    \frac{| \E_{\vx \sim (P^{\vpi})^t }\left[c(\vx)\right] -\E_{\vx\sim \bar{P}^{\vpi}} \left[c(\vx)\right]|}{1 + V^{\vpi}(\vx_0)} \leq C_2 \lambda^{t}.
\end{equation*}
In summary, 
\begin{equation*}
   | \E_{\vx \sim (P^{\vpi})^t }\left[c(\vx)\right] -\E_{\vx\sim \bar{P}^{\vpi}} \left[c(\vx)\right]| \leq C_2 (1 + V^{\vpi}(\vx_0))  \lambda^{t}.
\end{equation*}

Consider any $T \geq 0$, and define with $\Bar{c} = \E_{\vx\sim \bar{P}^{\vpi}} \left[c(\vx, \vpi(x))\right]$. 
\begin{align*}
\E_{\vpi} \left[ \sum^{T-1}_{t=0} c(\vx_t, \vu_t) - \Bar{c}\right] &= \sum^{T-1}_{t=0} \E_{(P^{\vpi})^t} \left[ c(\vx_t, \vu_t) \right]- \Bar{c} \\
    &\leq \sum^{T-1}_{t=0} \left|\E_{(P^{\vpi})^t} \left[ c(\vx_t, \vu_t) \right]- \Bar{c}\right| \\
    &\leq C_2 (1 + V^{\vpi}(\vx_0))\sum^{T-1}_{t=0} \lambda^t \\
    &=  C_2 (1 + V^{\vpi}(\vx_0)) \frac{1 - \lambda^T}{1 - \lambda} 
\end{align*}
Hence, we have 
\begin{equation*}
    \lim_{T \to \infty} \left|\E_{\vpi} \left[ \sum^{T-1}_{t=0} c(\vx_t, \vu_t) - \Bar{c}\right]\right|  \leq C_2 (1 + V^{\vpi}(\vx_0)) \frac{1}{1 - \lambda}, 
\end{equation*}
and for any $\vx_0$ in a compact subset of $\setX$
\begin{equation*}
    \lim_{T \to \infty} \frac{1}{T} \E_{\vpi} \left[ \sum^{T-1}_{t=0} c(\vx_t, \vu_t) - \Bar{c}\right]  = 0.
\end{equation*}
Moreover,
\begin{equation*}
    |B(\vpi, \vx_0)| \leq C_2(1 + V^{\vpi}(\vx_0)) \frac{1}{1 - \lambda}.
\end{equation*}
\end{proof}
Another interesting, inequality that follows from the proof above is the difference in bias inequality. 
 \begin{equation*}
            |\E_{\vx_0 \sim \zeta_1} [B(\vpi, \vx_0)] -  \E_{\vx_0 \sim \zeta_2} [B(\vpi, \vx_0)]| \leq  \frac{C_3}{1-\lambda} \int_{\setX} (1 + V^{\vpi}(\vx)) \left|\zeta_1 - \zeta_2 \right|(d\vx) 
         \end{equation*}
         for all probability measures $\zeta_1, \zeta_2$. 
To show this holds, define $C' = \max_{\vpi \in \Pi} \norm{c(\vx, \vpi(\vx))}_{1 + \theta V^{\vpi}}$. Furthermore, note that $C' < \infty$ from \cref{assumption: Stability} and  $\norm{\nicefrac{c(\vx, \vpi(\vx))}{C'}}_{1 + \theta V^{\vpi}} \leq 1$.
\begin{align*}
   &\left|\E_{\vx \sim (P^{\vpi})^t \zeta_1} c(\vx, \vpi(\vx))-\E_{\vx \sim (P^{\vpi})^t \zeta_2} c(\vx, \vpi(\vx))\right| = \left| \int_{\setX} c(\vx, \vpi(\vx)) ((P^{\vpi})^t \zeta_1 - (P^{\vpi})^t \zeta_2)(d\vx)\right| \\
    &= C' \left| \int_{\setX} \frac{1}{C'}c(\vx, \vpi(\vx)) ((P^{\vpi})^t \zeta_1 - (P^{\vpi})^t \zeta_2)(d\vx)\right| \\ 
    &\leq C' \sup_{\varphi: \norm{\varphi}_{1+\theta V^{\vpi}} \leq 1} \int_{\setX} \varphi(\vx) ((P^{\vpi})^t \zeta_1 - (P^{\vpi})^t \zeta_2)(d\vx) = C' \rho^{\vpi}_{\theta}((P^{\vpi})^t \zeta_1, (P^{\vpi})^t \zeta_2) \\
    &\leq C' \lambda \rho^{\vpi}_{\theta}((P^{\vpi})^{t-1} \zeta_1, (P^{\vpi})^{t-1} \zeta_2) \tag{\cref{thm:geometric_ergodicity}}\\
    &\leq C'\lambda^t \rho^{\vpi}_{\theta}(\zeta_1, \zeta_2).
\end{align*}
Also, note that there exists $C_{\theta} \in (0, \infty)$ such that
$C_{\theta}\norm{\varphi}_{1 + \theta V^{\vpi}} \geq \norm{\varphi}_{1 + V^{\vpi}}$ due to the equivalence of the two norms.
\begin{align*}
    \rho^{\vpi}_{\theta}(\zeta_1, \zeta_2) &= \sup_{\varphi: \norm{\varphi}_{1 + \theta V^{\vpi}} \leq 1} \int_{\setX} \varphi(\vx) (\zeta_1 - \zeta_2)(d\vx) \\
    &\leq \sup_{\varphi: \norm{\varphi}_{1 + V^{\vpi}} \leq C_{\theta}} \int_{\setX} \varphi(\vx) (\zeta_1 - \zeta_2)(d\vx) \\
    &= C_{\theta} \sup_{\varphi: \norm{\varphi}_{1 + V^{\vpi}} \leq 1} \int_{\setX} \varphi(\vx) (\zeta_1 - \zeta_2)(d\vx) \\
    &= C_{\theta} \rho^{\vpi}_{1}(\zeta_1, \zeta_2)
\end{align*}

Therefore, for the bias we have
\begin{align*}
    &\left|\E_{\vx_0 \sim \zeta_1} [B(\vpi, \vx_0)] -  \E_{\vx_0 \sim \zeta_2} [B(\vpi, \vx_0)]\right| \\
    &\leq \lim_{T\to \infty} \sum^{T-1}_{t=0} \left|\E_{\vx \sim (P^{\vpi})^t \zeta_1} c(\vx, \vpi(\vx))-\E_{\vx \sim (P^{\vpi})^t \zeta_2} c(\vx, \vpi(\vx))\right| \\
    &\leq  C' \rho^{\vpi}_{\theta}(\zeta_1, \zeta_2) \lim_{T\to \infty} \sum^{T-1}_{t=0} \lambda^t = \frac{C'}{1-\lambda} \rho^{\vpi}_{\theta}(\zeta_1, \zeta_2) \\
    &\leq \frac{C'C_{\theta}}{1-\lambda} \rho^{\vpi}_{1}(\zeta_1, \zeta_2) = \frac{C'C_{\theta}}{1-\lambda} \int_{\setX} (1 + V^{\vpi}(\vx)) \left|\zeta_1 - \zeta_2 \right|(d\vx) 
\end{align*}
Set $C_3 = C'C_{\theta}$.

\subsection{Proof of bounded average cost for the optimistic system}
In this section, we show that the results from \cref{thm: existence of average cost problem} also transfer over to the optimistic dynamics.
\begin{theorem}[Existence of Average Cost Solution for the Optimistic System]
    Let \cref{ass:lipschitz_continuity} -- \ref{ass:rkhs_func} hold. Consider any $n> 0$ and
    let $\vpi_n, \vf_n$ denote the solution to \cref{eq:exploration_op_optimistic}, 
    $P^{\vpi, \vf_n}$ its transition kernel.  
    Then $P^{\vpi, \vf_n}$ admits a unique invariant measure $\bar{P}^{\vpi_n, \vf_n}$  and there exists
    $C_2, C_3 \in (0, \infty)$, $\hat{\lambda} \in (0, 1)$ such that
    \begin{itemize}[leftmargin=*]
        \item[] \emph{Average Cost};
        \begin{equation*}
A(\vpi_n, \vf_n) = \lim_{T \to \infty} \frac{1}{T}\E_{\vpi_n, \vf_n} \left[ \sum^{T-1}_{t=0} c(\vx_t, \vu_t) \right] = \E_{\vx\sim \bar{P}^{\vpi_n, \vf_n}} \left[c(\vx, \vpi_n(x))\right]
        \end{equation*}
        \item [] \emph{Bias Cost};
        \begin{equation*}
 |B(\vpi_n, \vf_n, \vx_0)| =  \left|\lim_{T \to \infty} \E_{\vpi_n, \vf_n} \left[ \sum^{T-1}_{t=0} c(\vx_t, \vu_t) -  A(\vpi_n, \vf_n) \right]\right| \leq C_2 (1 + V^{\vpi_n}(\vx_0)) \frac{1}{1 - \hat{\lambda}}
        \end{equation*}
        for all $\vx_0 \in \setX$.
        \item [] \emph{Difference in Bias};
        \begin{equation*}
            |\E_{\vx_0 \sim \zeta_1} [B(\vpi_n, \vf_n, \vx_0)] -  \E_{\vx_0 \sim \zeta_2} [B(\vpi_n, \vf_n, \vx_0)]| \leq  \frac{C_3}{1-\hat{\lambda}} \int_{\setX} (1 +  V^{\vpi}(\vx)) \left|\zeta_1 - \zeta_2 \right|(d\vx) 
        \end{equation*}
        for all probability measures $\zeta_1, \zeta_2$.
    \end{itemize}
    \label{thm: existence of average cost problem hallucinated system}
\end{theorem}
\Cref{thm: existence of average cost problem hallucinated system} shows that the optimistic dynamics $\vf_n$ 
retain the boundedness property from the true dynamics $\vf^*$ and give a well-defined solution \wrt average cost and the bias cost. 
To prove \cref{thm: existence of average cost problem hallucinated system} we show that the optimistic system also satisfies the drift and minorisation condition. Then we can invoke the result from \cite{hairer2011yet} similar to the proof of \cref{thm: existence of average cost problem}.
\begin{lemma}[Stability of optimistic system]
    Let \cref{ass:lipschitz_continuity} -- \ref{ass:rkhs_func} hold, then we have with probability at least $1-\delta$ for all 
    $n \geq 0$,
    $\vpi \in \Pi$, $\vf \in \setM_n \cap \setM_0$, that there exists a constant $\widehat{K} > 0$ such that
    \begin{equation*}
        \E_{\vx_+|\vx, \vf, \vpi}[V^{\vpi}(\vx_+)] \leq \gamma V^{\vpi}(\vx) + \widehat{K},
    \end{equation*}
    where $\vx_+ = \vf(\vx, \vpi(\vx) + \vw$.
    \label{lem: stability of hallucinated system}
\end{lemma}
\begin{proof}
Note, that $V^{\vpi}$ is uniformly continuous w.r.t.~$\kappa$
\begin{equation*}
    \left| V^{\vpi}(\vx) - V^{\vpi}(\vx') \right| \leq \kappa(\norm{\vx - \vx'}). 
\end{equation*}

Furthermore, since $\vf \in \setM_n \cap \setM_0$ and therefore $\vf \in \setM_0$, we have that there exists some $\veta \in [-1, 1]^{dx}$ such that
\begin{equation*}
    \vf(\vx, \vpi(\vx)) = \vmu_0(\vx. \vpi(\vx)) + \beta_0 \vsigma_0(\vx, \vpi(\vx)) \veta(\vx).
\end{equation*}

\begin{align*}
    \E_{\vw}&[V^{\vpi}(\vmu_0(\vx. \vpi(\vx)) + \beta_0 \vsigma_0(\vx, \vpi(\vx)) \veta(\vx) + \vw)] - \E_{\vw}[V^{\vpi}(\vf^*(\vx. \vpi(\vx)) + \vw)] \\
    &\leq 
    \kappa\left( \norm{\vmu_0(\vx. \vpi(\vx)) + \beta_0 \vsigma_0(\vx, \vpi(\vx)) \veta(\vx) - \vf^*(\vx. \vpi(\vx))}\right) \\
    &\leq \kappa \left(\norm{\vmu_0(\vx. \vpi(\vx)) - \vf^*(\vx. \vpi(\vx))} + \norm{\beta_0 \vsigma_0(\vx, \vpi(\vx))\veta(\vx)}\right)\\
    &\leq \kappa \left(\left(1+\sqrt{d_x}\right) \beta_0 \sqrt{d_x}\sigma_{\max}\right) \tag{\cref{ass:rkhs_func}}.
\end{align*}
Therefore, 
\begin{align*}
    \E_{\vx_+|\vx, \vf, \vpi}[V^{\vpi}(\vx_+)] &\leq \E_{\vx_+|\vx, \vf^*, \vpi}[V^{\vpi}(\vx_+^*)] + \kappa\left(\left(1+\sqrt{d_x}\right) \beta_0 \sqrt{d_x}\sigma_{\max}\right)\\
    &= \E_{\vx_+|\vx, \vf^*, \vpi}[V^{\vpi}(\vx_+^*)] + \kappa\left(\left(1+\sqrt{d_x}\right) \beta_0 \sqrt{d_x}\sigma_{\max}\right) \\
    &\leq \gamma V^{\vpi}(\vx) + K + \kappa\left(\left(1+\sqrt{d_x}\right) \beta_0 \sqrt{d_x}\sigma_{\max}\right) \tag{\cref{assumption: Stability}},
\end{align*}
where we denoted $\vx_+^* = \vf^*(\vx, \vpi(\vx) + \vw$.
Define $\widehat{K} = K + \kappa\left(\left(1+\sqrt{d_x}\right) \beta_0 \sqrt{d_x}\sigma_{\max}\right)$.
\end{proof}

\begin{lemma}[Minorisation condition optimistic system]
\label{lem: minorisation condition hallucinated system}
Consider the system
\begin{equation*}
    \vx_+ = \vf(\vx. \vpi(\vx)) + \vw
\end{equation*}
for any $n \geq 0$, $\vpi \in \Pi$ and $\vf \in \setM_n \cap \setM_0$. 
Let \cref{ass:lipschitz_continuity} -- \ref{ass:rkhs_func} hold. Let $P^{\vpi, \vf}$ denote the transition kernel for the policy $\vpi \in \Pi$ i.e., 
$P^{\vpi, \vf}(\vx, \setA) = \Pr{\vx_+ \in \setA |\vx, \vpi(\vx), \vf}$. Then, there  exists a constant $\hat{\alpha} \in (0, 1)$ and a probability measure $\hat{\zeta}(\cdot)$ independent of $n$ s.t., 
\begin{equation}
    \inf_{\vx \in \setC} P^{\vpi, \vf}(\vx, \cdot) \geq \hat{\alpha} \hat{\zeta}(\cdot) \label{eq:minorisation condition optimistic system}
\end{equation}
with $\setC \defeq \{\vx \in \setX; V^{\vpi}(\vx) < \hat{R}\}$ for some $\hat{R} >  \nicefrac{2 \widehat{K}}{1 - \gamma}$
\end{lemma}
\begin{proof}
    First, we show that $\setC$ is contained in a compact domain. From the \Cref{assumption: Stability} we pick the function $\xi \in \setK_{\infty}$. Since $C_l\xi(0) = 0, \lim_{s \to \infty} \xi(s) = +\infty$ and $C_l \xi$ is continuous, there exists $M$ such that $C_l \xi(M) = \hat{R}$. Then for $\norm{\vx} > M$ we have:
    \begin{align*}
        V^{\vpi}(\vx) \ge C_l\xi(\norm{\vx}) > \xi(M) = \hat{R}.
    \end{align*}
    Therefore we have: $\setC \subseteq \setB(\vzero, M) \defeq \{\vx \mid \norm{\vx - \vzero} \le M\}$. Since for any $\vx \in \setC$ we have $\norm{\vf(\vx, \vpi(\vx))} \le \norm{\vf^*(\vx, \vpi(\vx))} + \beta_0\sigma_{\max}$. Since $\vf^*$ is continuous, there exists a $B$ such that $\vf^*(\setC, \vpi(\setC)) \subset \setB(\vzero, B)$. Therefore we have: $\vf(\setC, \vpi(\setC)) \subset \setB(\vzero,B_1)$, where $B_1= B + \beta_0\sigma_{\max}$.
    In the last step we prove that $\alpha \defeq 2^{-d_{\vx}}e^{-B_1^2/\sigma^2}$ and $\zeta$ with law of $\setN\left(0, \frac{\sigma^2}{2}\right)$ satisfy condition of \Cref{lem: minorisation condition}.
    It is enough to show that $\forall \vmu \in \setB(\vzero, B_1), \forall \vx \in \R^{d_\vx}$ we have:
    \begin{align*}
        \alpha \frac{1}{(2\pi)^{\frac{d_x}{2}}\left(\frac{\sigma^2}{2}\right)^{\frac{d_\vx}{2}}}e^{-\frac{\norm{\vx}^2}{\sigma^2}} \le \frac{1}{(2\pi)^{\frac{d_x}{2}}(\sigma^2)^{\frac{d_\vx}{2}}}e^{-\frac{\norm{\vx - \vmu}^2}{2\sigma^2}}
    \end{align*}
    which can be proven with simple algebraic manipulations.
\end{proof}

\begin{proof}[Proof of \cref{thm: existence of average cost problem hallucinated system}]
As for the true system, the drift condition from~\cref{lem: stability of hallucinated system} and the minorisation condition from~\cref{lem: minorisation condition hallucinated system} are sufficient to show ergodicity of the optimistic system (c.f.,~\cref{thm:geometric_ergodicity} or \cite{hairer2011yet}). The rest of the proof is similar to~\cref{thm: existence of average cost problem}.
\end{proof}

\subsection{Proof of \cref{thm: Regret bound neo rl}}
Since \neorl works in artificial episodes $n \in \{0, N-1\}$ of varying horizons $H_n$. We denote with $\vx^n_k$ the state visited during episode $n$ at time step $k \leq H_n$. Crucial, to our regret analysis is bounding the first and second moment of $V^{\vpi_n}(\vx^n_k)$ for all $n, k$. Given the nature of \cref{assumption: Stability}, this requires analyzing geometric series. Thus, we start with the following elementary result of geometric series.
\begin{corollary}
Consider the sequence $\{S_n\}_{n\geq 0}$ with $S_n \geq 0$ for all $n$. Let the following hold
\begin{equation*}
    S_n \leq \rho S_{n-1} + C
\end{equation*}
for $\rho \in (0, 1)$ and $C>0$. Then we have
\begin{equation*}
    S_n \leq \rho^n S_0 + C \frac{1}{1-\rho}.
\end{equation*}
\label{cor:geometric_series_bound}
\end{corollary}
\begin{proof}
    \begin{equation*}
        S_n \leq \rho S_{n-1} + C 
        \leq  \rho^2 S_{n-2} + C  (1 + \rho) 
        \leq \rho^n S_{0} + C\sum^n_{i=0} \rho^i \leq \rho^n S_0 + C \frac{1}{1-\rho}.
    \end{equation*}
\end{proof}
\begin{lemma}
 Let \cref{ass:lipschitz_continuity} -- \ref{ass:rkhs_func} hold and let $H_0$ be the smallest integer such that
    \begin{equation*}
        H_0 > \frac{\log\left(\sfrac{C_u}{C_l}\right)}{\log\left(\sfrac{1}{\gamma}\right)}.
    \end{equation*}
    Moreover, define $\nu = \frac{C_u}{C_l}\gamma^{H_0}$. Note, by definition of $H_0$, $\nu < 1$.
    Then we have for all $k \in \{0, \dots, H_n\}$ and $n > 0$ 
    \begin{itemize}[leftmargin=*]
        \item[]\emph{Bounded expectation over horizon}
        \begin{equation}
            \E_{\vx^n_k, \dots, \vx^0_1|\vx_0}[V^{\vpi_n}(\vx^n_k)] \leq \gamma^{k}\E_{\vx^n_{0}, \dots, \vx^0_1|\vx_0}[V^{\vpi_n}(\vx^n_0)] + K/(1-\gamma).
        \label{eq:bound_over_horizon_in_expectation}
        \end{equation}
        \item[]\emph{Bounded expectation over episodes}
          \begin{equation}
             \E_{\vx^n_0, \dots, \vx^0_1|\vx_0}[V^{\vpi_n}(\vx^n_0)]\leq \nu^{n}V^{\vpi_{0}}(\vx_0) + \frac{C_u}{C_l}K/(1-\gamma)\frac{1}{1-\nu}.
             \label{eq:bound_over_episodes_in_expectation}
          \end{equation}
    \end{itemize}
    Moreover, we have 
    \begin{equation}
        \E_{\vx^n_k, \dots, \vx^0_1|\vx_0} [V^{\vpi_n}(\vx^n_k)] \leq  D(\vx_0, K, \gamma, \nu) ,
        \label{eq:bounded expectation}
    \end{equation}
    with $D(\vx_0, K, \gamma, \nu) = V^{\vpi_{0}}(\vx_0) + K/(1-\gamma) \left(\frac{C_u}{C_l}\frac{1}{1-\nu} + 1\right)$
    \label{lem:bounded first moment}
\end{lemma}
\begin{proof}
We start with proving the first claim
\begin{align*}
     \E_{\vx^n_k, \dots, \vx^0_1|\vx_0}[V^{\vpi_n}(\vx^n_k)] &= \E_{\vx^n_{k-1}, \dots, \vx^0_1|\vx_0}[\E_{\vx^n_{k}|\vx^n_{k-1}}[V^{\vpi_n}(\vx^n_k)]] \\
     &\leq \E_{\vx^n_{k-1}, \dots, \vx^0_1|\vx_0}[\gamma V^{\vpi_n}(\vx^n_{k-1}) + K] \tag{\cref{assumption: Stability}}\\
     &= \gamma \E_{\vx^n_{k-1}, \dots, \vx^0_1|\vx_0}[ V^{\vpi_n}(\vx^n_{k-1})] + K
\end{align*}
 We can apply \cref{cor:geometric_series_bound} to prove the claim.
    For the second claim, we note that for any $\vpi, \vpi'$ and $\vx \in \setX$ we have from \cref{ass:Policy class}
    \begin{equation*}
        V^{\vpi}(\vx) \leq C_u \alpha(\norm{\vx}) \leq 
        \frac{C_u}{C_l} V^{\vpi'}(\vx).
    \end{equation*}

    Therefore,
    \begin{align*}
        &\E_{\vx^n_0, \dots, \vx^0_1|\vx_0}[V^{\vpi_n}(\vx^n_0)] \\
        &\leq \frac{C_u}{C_l} \E_{\vx^n_0, \dots, \vx^0_1|\vx_0}[V^{\vpi_{n-1}}(\vx^{n}_{0})] \\
        &= \frac{C_u}{C_l} \E_{\vx^{n-1}_{H_{n}}, \dots, \vx^0_1|\vx_0}[V^{\vpi_{n-1}}(\vx^{n-1}_{H_{n}})] \tag{Since $\vx^{n}_{0} = \vx^{n-1}_{H_{n}}$}\\
        &\leq  \left(\frac{C_u}{C_l} \gamma^{H_{n}}\right)\E_{\vx^{n-1}_{0}, \dots, \vx^0_1|\vx_0}[V^{\vpi_{n-1}}(\vx^{n-1}_0)] + \frac{C_u}{C_l}K/(1-\gamma) \tag{\Cref{eq:bound_over_horizon_in_expectation}}
    \end{align*}
    For our choice of $H_0$, we have for all $n\geq 0$ that $\frac{C_u}{C_l} \gamma^{H_{n}} \leq  \frac{C_u}{C_l} \gamma^{H_{0}} \leq \nu < 1$. From \cref{cor:geometric_series_bound}, we get
     \begin{align*}
        \E_{\vx^n_0, \dots, \vx^0_1|\vx_0}[V^{\vpi_n}(\vx^n_0)] &\leq \left(\frac{C_u}{C_l} \gamma^{H_{n}}\right)\E_{\vx^{n-1}_{0}, \dots, \vx^0_1|\vx_0}[V^{\vpi_{n-1}}(\vx^{n-1}_0)] + \frac{C_u}{C_l}K/(1-\gamma) \\
        &\leq \nu \E_{\vx^{n-1}_{0}, \dots, \vx^0_1|\vx_0}[V^{\vpi_{n-1}}(\vx^{n-1}_0)] + \frac{C_u}{C_l}K/(1-\gamma) \\
                &\leq \nu^{n}V^{\vpi_{0}}(\vx_0) + \frac{C_u}{C_l}K/(1-\gamma)\frac{1}{1-\nu}. \tag{\cref{cor:geometric_series_bound}}
    \end{align*}

    \begin{align*}
         \E_{\vx^n_k, \dots, \vx^0_1|\vx_0} [V^{\vpi_n}(\vx^n_k)] &\leq \gamma^{k}\E_{\vx^n_{0}, \dots, \vx^0_1|\vx_0}[V^{\vpi_n}(\vx^n_0)] + K/(1-\gamma) \tag{\cref{eq:bound_over_horizon_in_expectation}}\\
         &\leq \E_{\vx^n_{0}, \dots, \vx^0_1|\vx_0}[V^{\vpi_n}(\vx^n_0)] + K/(1-\gamma) \\
         &\leq \nu^{n}V^{\vpi_{0}}(\vx_0) + \frac{C_u}{C_l}K/(1-\gamma)\frac{1}{1-\nu} + K/(1-\gamma) \tag{\cref{eq:bound_over_episodes_in_expectation}}\\
         &\leq  V^{\vpi_{0}}(\vx_0) + \frac{C_u}{C_l}K/(1-\gamma)\frac{1}{1-\nu} + K/(1-\gamma)
    \end{align*}
\end{proof}
\begin{lemma}
 Let \cref{ass:lipschitz_continuity} -- \ref{ass:rkhs_func} hold and let $H_0$ be the smallest integer such that
    \begin{equation*}
        H_0 > \frac{\log\left(\sfrac{C_u}{C_l}\right)}{\log\left(\sfrac{1}{\gamma}\right)}.
    \end{equation*}
    Moreover, define $\nu = \frac{C_u}{C_l}\gamma^{H_0}$. Note, by definition of $H_0$, $\nu < 1$.
    
Then we have for all $k \in \{0, \dots, H_n\}$ and $n > 0$ 
    \begin{itemize}[leftmargin=*]
        \item[]\emph{Bounded second moment over horizon}
    \begin{equation}
        \E_{\vx^n_k, \dots, \vx^0_1|\vx_0}\left[\left(V^{\vpi_n}(\vx^n_k)\right)^2\right] \leq  \gamma^{2k}  \E_{\vx^n_{0}, \dots, \vx^0_1|\vx_0}\left[\left(V^{\vpi_n}(\vx^n_{0})\right)^2\right] + \frac{D_2(\vx_0, K, \gamma, \nu) }{1-\gamma^2}
        \label{eq:bounded_over_horizon_second_moment}
    \end{equation}
    with $ D_2(\vx_0, K, \gamma, \nu) = 2 K \gamma D(\vx_0, K, \gamma, \nu) + K^2 + C_{\vw}$, and $C_{\vw} = \E_{\vw} \left[\kappa^2(\norm{w})\right] + 3 (\E_{\vw} \left[\kappa(\norm{w})\right])^2$.
 \item[]\emph{Bounded second moment over episodes}
    \begin{equation}
         \E_{\vx^n_0, \dots, \vx^0_1|\vx_0}\left[\left(V^{\vpi_n}(\vx^n_0)\right)^2\right] \leq \nu^{2n} \left(V^{\vpi_{0}}(\vx_{0})\right)^2 + \left(\frac{C_u}{C_l}\right)^2 \frac{D_2(\vx_0, K, \gamma, \nu) }{1-\gamma^2}\frac{1}{1-\nu^{2}}.
         \label{eq:bounded_over_episodes_second_moment}
    \end{equation}
    \end{itemize}
    Moreover, let $D_3(\vx_0, K, \gamma, \nu) = \left(V^{\vpi_{0}}(\vx_{0})\right)^2 + D_2(\vx_0, K, \gamma, \nu) \left(\left(\frac{C_u}{C_l}\right)^2 \frac{1}{1-\gamma^2}\frac{1}{1-\nu^{2}} + \frac{1}{1-\gamma^2}\right)$.

    \begin{equation*}
         \E_{\vx^n_k, \dots, \vx^0_1|\vx_0}\left[\left(V^{\vpi_n}(\vx^n_k)\right)^2\right] \leq  D_3(\vx_0, K, \gamma, \nu) 
         \label{eq:bounded_second_moment}
    \end{equation*}
    \label{lem:bounded_second_moment}
\end{lemma}

\begin{proof}
Note that,
    \begin{align*}
        \E_{\vx^n_k|\vx^n_{k-1}}\left[\left(V^{\vpi_n}(\vx^n_k)\right)^2\right] &= \left(\E_{\vx^n_k|\vx^n_{k-1}}\left[V^{\vpi_n}(\vx^n_k)\right]\right)^2 \\ 
        &+  \E_{\vx^n_k|\vx^n_{k-1}}\left[\left(V^{\vpi_n}(\vx^n_k) - \E_{\vx^n_k|\vx^n_{k-1}}\left[V^{\vpi_n}(\vx^n_k)\right] \right)^2\right].
    \end{align*}
We first bound the second term. Let $\bar{\vx}^n_k = \vf^*(\vx^n_{k-1}, \vpi_n(\vx^n_{k-1}))$, i.e., the next state in the absence of transition noise.
\begin{align*}
    \E_{\vx^n_k|\vx^n_{k-1}} &\left[\left(V^{\vpi_n}(\vx^n_k) - \E_{\vx^n_k|\vx^n_{k-1}}\left[V^{\vpi_n}(\vx^n_k)\right] \right)^2\right] \\
    &= \E_{\vx^n_k|\vx^n_{k-1}}\left[\left(V^{\vpi_n}(\vx^n_k) - V^{\vpi_n}(\bar{\vx}^n_k) + V^{\vpi_n}(\bar{\vx}^n_k) -\E_{\vx^n_k|\vx^n_{k-1}}\left[V^{\vpi_n}(\vx^n_k)\right] \right)^2\right] \\
    &=  \E_{\vx^n_k|\vx^n_{k-1}}\left[\left(V^{\vpi_n}(\vx^n_k) - V^{\vpi_n}(\bar{\vx}^n_k) + \E_{\vx^n_k|\vx^n_{k-1}}\left[V^{\vpi_n}(\bar{\vx}^n_k)-V^{\vpi_n}(\vx^n_k)\right] \right)^2\right] \\
    &\leq \E_{\vw} \left[\left(\kappa(\norm{w}) + \E_{\vw}[\kappa(\norm{w})]\right)^2\right] \tag{uniform continuity of $V^{\vpi_n}$}\\
    &= \E_{\vw} \left[\kappa^2(\norm{\vw})\right] + 3 (\E_{\vw} \left[\kappa(\norm{\vw})\right])^2 \\
    &= C_{\vw} \tag{\cref{assumption: Stability}}
\end{align*}

Therefore we have
\begin{align*}
        \E_{\vx^n_k|\vx^n_{k-1}}\left[\left(V^{\vpi_n}(\vx^n_k)\right)^2\right] &= \left(\E_{\vx^n_k|\vx^n_{k-1}}\left[V^{\vpi_n}(\vx^n_k)\right]\right)^2 +  C_{\vw} \\
        &\leq \left(\gamma V^{\vpi_n}(\vx^n_k) + K\right)^2 +  C_{\vw} \\
        &=\gamma^2 \left(V^{\vpi_n}(\vx^n_{k-1})\right)^2 + 2 K \gamma V^{\vpi_n}(\vx^n_{k-1}) + K^2 + C_{\vw}.
    \end{align*}

\begin{align*}
    &\E_{\vx^n_k, \dots, \vx^0_1|\vx_0}\left[\left(V^{\vpi_n}(\vx^n_k)\right)^2\right] \\ 
    &= \E_{\vx^n_{k-1}, \dots, \vx^0_1|\vx_0}\left[\E_{\vx^n_k|\vx^n_{k-1}}\left[\left(V^{\vpi_n}(\vx^n_k)\right)^2\right]\right]\\
    &\leq \gamma^2 \E_{\vx^n_{k-1}, \dots, \vx^0_1|\vx_0}\left[\left(V^{\vpi_n}(\vx^n_{k-1})\right)^2\right] + 2 K \gamma \E_{\vx^n_{k-1}, \dots, \vx^0_1|\vx_0}\left[V^{\vpi_n}(\vx^n_{k-1})\right] + K^2 + C_{\vw} \\
    &\leq \gamma^2 \E_{\vx^n_{k-1}, \dots, \vx^0_1|\vx_0}\left[\left(V^{\vpi_n}(\vx^n_{k-1})\right)^2\right] + 2 K \gamma D(\vx_0, K, \gamma, \nu) + K^2 + C_{\vw} \tag{\cref{lem:bounded first moment}}.
\end{align*}
Let $D_2(\vx_0, K, \gamma, \nu) = 2 K \gamma D(\vx_0, K, \gamma, \nu) + K^2 + C_{\vw}$. 
    Applying \cref{cor:geometric_series_bound} we get
    \begin{equation*}
        \E_{\vx^n_k, \dots, \vx^0_1|\vx_0}\left[\left(V^{\vpi_n}(\vx^n_k)\right)^2\right] \leq  \gamma^{2k}  \E_{\vx^n_{0}, \dots, \vx^0_1|\vx_0}\left[\left(V^{\vpi_n}(\vx^n_{0})\right)^2\right] + \frac{D_2(\vx_0, K, \gamma, \nu) }{1-\gamma^2}
    \end{equation*}

Similar to the first moment, 
we leverage that $V^{\vpi_n}(\vx) \leq \frac{C_u}{C_l} V^{\vpi_{n-1}}(\vx)$ for all $\vx \in \setX$, $\frac{C_u}{C_l}\gamma^{H_{n-1}} \leq \nu$,
and get,
\begin{align*}
    &\E_{\vx^n_0, \dots, \vx^0_1|\vx_0}\left[\left(V^{\vpi_n}(\vx^n_0)\right)^2\right] \\ 
    &\leq \left(\frac{C_u}{C_l}\right)^2\E_{\vx^n_0, \dots, \vx^0_1|\vx_0}\left[\left(V^{\vpi_{n-1}}(\vx^{n}_{0})\right)^2\right] \\
    &=\left(\frac{C_u}{C_l}\right)^2\E_{\vx^{n-1}_{H_{n}}, \dots, \vx^0_1|\vx_0}\left[\left(V^{\vpi_{n-1}}(\vx^{n-1}_{H_{n}})\right)^2\right] \tag{Since $\vx^{n}_{0} = \vx^{n-1}_{H_{n}}$}\\
    &\leq \left(\frac{C_u}{C_l}\gamma^{H_{n}}\right)^2 \E_{\vx^{n-1}_{0}, \dots, \vx^0_1|\vx_0}\left[\left(V^{\vpi_{n-1}}(\vx^{n-1}_{0})\right)^2\right] + \left(\frac{C_u}{C_l}\right)^2 \frac{D_2(\vx_0, K, \gamma, \nu) }{1-\gamma^2} \tag{\cref{eq:bounded_over_horizon_second_moment}}\\
    &\leq \nu^2 \E_{\vx^{n-1}_{0}, \dots, \vx^0_1|\vx_0}\left[\left(V^{\vpi_{n-1}}(\vx^{n-1}_{0})\right)^2\right] + \left(\frac{C_u}{C_l}\right)^2 \frac{D_2(\vx_0, K, \gamma, \nu) }{1-\gamma^2} \\
    &\leq \nu^{2n} \left(V^{\vpi_{0}}(\vx_{0})\right)^2 + \left(\frac{C_u}{C_l}\right)^2 \frac{D_2(\vx_0, K, \gamma, \nu) }{1-\gamma^2}\frac{1}{1-\nu^{2}} \tag{\cref{cor:geometric_series_bound}}
\end{align*}
Moreover,
\begin{align*}
     &\E_{\vx^n_k, \dots, \vx^0_1|\vx_0}\left[\left(V^{\vpi_n}(\vx^n_k)\right)^2\right] \\ 
     &\leq 
     \gamma^{2k}\E_{\vx^n_0, \dots, \vx^0_1|\vx_0}\left[\left(V^{\vpi_n}(\vx^n_0)\right)^2\right] + \frac{D_2(\vx_0, K, \gamma, \nu) }{1-\gamma^2} \tag{\cref{eq:bounded_over_horizon_second_moment}} \\
     &\leq \E_{\vx^n_0, \dots, \vx^0_1|\vx_0}\left[\left(V^{\vpi_n}(\vx^n_0)\right)^2\right] + \frac{D_2(\vx_0, K, \gamma, \nu) }{1-\gamma^2} \\
     &\leq \nu^{2n} \left(V^{\vpi_{0}}(\vx_{0})\right)^2 + \left(\frac{C_u}{C_l}\right)^2 \frac{D_2(\vx_0, K, \gamma, \nu) }{1-\gamma^2}\frac{1}{1-\nu^{2}}  + \frac{D_2(\vx_0, K, \gamma, \nu) }{1-\gamma^2} \tag{\cref{eq:bounded_over_episodes_second_moment}} \\  
     &\leq \left(V^{\vpi_{0}}(\vx_{0})\right)^2 + D_2(\vx_0, K, \gamma, \nu) \left(\left(\frac{C_u}{C_l}\right)^2 \frac{1}{1-\gamma^2}\frac{1}{1-\nu^{2}} + \frac{1}{1-\gamma^2}\right) 
\end{align*}

\end{proof}

Finally, we prove the regret bound of \neorl.
\begin{proof}[Proof of \cref{thm: Regret bound neo rl}]
In the following, let $\hat{\vx}^{n}_{k+1}=\vf_n(\vx^{n}_{k}, \vpi_n(\vx^{n}_{k})) + \vw^n_k$ denote the state predicted under the optimistic dynamics and $\vx^{n}_{k+1}=\vf^*_n(\vx^{n}_{k}, \vpi_n(\vx^{n}_{k})) + \vw^n_k$ the true state.
    \begin{align*}
        &\E\left[\sum^{N-1}_{n=0} \sum^{H_n-1}_{k=0} c(\vx^{n}_k, \vpi_n(\vx^{n}_k)) - A(\vpi^*)\right] \\
        &\leq \E\left[\sum^{N-1}_{n=0} \sum^{H_n-1}_{k=0}\ c(\vx^{n}_k, \vpi_n(\vx^{n}_k)) - A(\vpi_n, \vf_n) \right] \tag{Optimism} \\
        &= \E\left[\sum^{N-1}_{n=0} \sum^{H_n-1}_{k=0}  B(\vpi_n, \vf_n, \vx^{n}_{k}) - B(\vpi_n, \vf_n, \hat{\vx}^{n}_{k+1})\right] \tag{Bellman equation (~\cref{eq: Bellman Equation Average Cost})} \\
        &= \E\left[\sum^{N-1}_{n=0} \sum^{H_n-1}_{k=0} 
        B(\vpi_n, \vf_n, \vx^{n}_{k}) - B(\vpi_n, \vf_n, \vx^{n}_{k+1}) + B(\vpi_n, \vf_n, \vx^{n}_{k+1}) - B(\vpi_n, \vf_n, \hat{\vx}^{n}_{k+1})\right] \\
        &= \sum^{N-1}_{n=0} \sum^{H_n-1}_{k=0} \E\left[B(\vpi_n, \vf_n, \vx^{n}_{k+1}) - B(\vpi_n, \vf_n, \hat{\vx}^{n}_{k+1}) \right] \tag{A} \\ 
        &+  \sum^{N-1}_{n=0} \sum^{H_n-1}_{k=0} \E\left[ B(\vpi_n, \vf_n, \vx^{n}_{k}) - B(\vpi_n, \vf_n, \vx^{n}_{k+1})\right] \tag{B}
    \end{align*}
First, we study the term (A).

\textbf{Proof for (A)}:
Note that because $\vf_n \in \setM_n$, there exists a $\veta \in [-1, 1]^{d_x}$ such that
$\hat{\vx}^{n}_{k+1} = \vmu_n(\vx^{n}_{k}, \vpi_n(\vx^{n}_{k})) + \beta_n \vsigma_n(\vx^{n}_{k}, \vpi_n(\vx^{n}_{k})) \veta(\vx^{n}_k) + \vw^n_k$. Furthermore, $\vx^{n}_{k+1} = \vf^*(\vx^{n}_{k}, \vpi_n(\vx^{n}_{k})) + \vw^n_k$ and the transition noise is Gaussian. Let $\zeta^n_{2, k}$ and $\zeta^n_{1, k}$ denote the respective distributions of the two random variables, i.e., $\zeta^n_{1, k} \sim \setN(\vf^*(\vx^{n}_{k}, \vpi_n(\vx^{n}_{k})), \sigma^2\mI)$ and $\zeta^n_{2, k} \sim \setN(\vf_n(\vx^{n}_{k}, \vpi_n(\vx^{n}_{k})), \sigma^2\mI)$. 
Next, define $\bar{B} = \E_{\vx \sim \zeta^n_{2, k}} \left[B(\vpi_n, \vf_n, \vx)\right]$, and consider the function $h(\vx) = B(\vpi_n, \vf_n, \vx) - \Bar{B}$. Then we have
\begin{align*}
    \E_{\vw^{n}_{k}}&\left[B(\vpi_n, \vf_n, \vx^{n}_{k+1}) - B(\vpi_n, \vf_n, \hat{\vx}^{n}_{k+1}) \right] \\ 
    &= \E_{\vx \sim \zeta^n_{1, k}} \left[B(\vpi_n, \vf_n, \vx)\right] - \E_{\vx \sim \zeta^n_{2, k}} \left[B(\vpi_n, \vf_n, \vx)\right] \\
    &= \E_{\vx \sim \zeta^n_{1, k}} \left[B(\vpi_n, \vf_n, \vx) - \Bar{B}\right] - \E_{\vx \sim \zeta^n_{2, k}} \left[B(\vpi_n, \vf_n, \vx) - \Bar{B}\right] \\
    &= \E_{\vx \sim \zeta^n_{1, k}}[h(\vx)] - \E_{\vx \sim \zeta^n_{2, k}}[h(\vx)].
\end{align*}
Note that $\E_{\vx \sim \zeta^n_{2, k}}[h(\vx)] = 0$ by the definition of $h$ and thus, 
\begin{equation}
\label{eq: kakade bound part 1}
    \E_{\vx \sim \zeta^n_{1, k}}[h(\vx)] - \E_{\vx \sim \zeta^n_{2, k}}[h(\vx)] = \E_{\vx \sim \zeta^n_{1, k}}[h(\vx)] \leq \sqrt{\E_{\vx \sim \zeta^n_{1, k}}[h^2(\vx)]}.
\end{equation}
In the following, we bound the term above \wrt the Chi-squared distance
\begin{align*}
    \E_{\vw^{n}_{k}}&\left[B(\vpi_n, \vf_n, \vx^{n}_{k+1}) - B(\vpi_n, \vf_n, \hat{\vx}^{n}_{k+1}) \right] = \E_{\vx \sim \zeta^n_{1, k}}[h(\vx)] - \E_{\vx \sim \zeta^n_{2, k}}[h(\vx)] \\
   &=  \int_{\setX} h(\vx) \left(1 - \frac{\zeta^n_{2, k}}{\zeta^n_{1, k}}\right) \zeta^n_{1, k}(d\vx)
   \leq \sqrt{\E_{\vx \sim \zeta^n_{1, k}}\left[h^2(\vx)\right]} \sqrt{d_{\chi}(\zeta^n_{2, k}, \zeta^n_{1, k})} \tag{\cite[Lemma C.2.,]{kakade2020information}}
\end{align*}
    With $d_{\chi}(\zeta^n_{2, k}, \zeta^n_{1, k})$ being the Chi-squared distance.
    \begin{equation*}
        d_{\chi}(\zeta^n_{2, k}, \zeta^n_{1, k}) = \int_{\setX}\frac{ \left( \zeta^n_{1, k} - \zeta^n_{2, k}\right)^2}{\zeta^n_{1, k}} (d\vx)
    \end{equation*}
Since both bounds from \Cref{eq: kakade bound part 1} and bound we got by applying \cite[Lemma C.2.,]{kakade2020information}, we can apply minimum and have:
\begin{equation*}
    \E_{\vw^{n}_{k}}\left[B(\vpi_n, \vf_n, \vx^{n}_{k+1}) - B(\vpi_n, \vf_n, \hat{\vx}^{n}_{k+1}) \right] \leq \sqrt{\E_{\vx \sim \zeta^n_{1, k}}\left[h^2(\vx)\right]} \sqrt{\min\left\{d_{\chi}(\zeta^n_{2, k}, \zeta^n_{1, k}), 1\right\}}
\end{equation*}
Therefore, following \citet[Lemma C.2.,]{kakade2020information} we get
\begin{align*}
    \E_{\vw^{n}_{k}}&\left[B(\vpi_n, \vf_n, \vx^{n}_{k+1}) - B(\vpi_n, \vf_n, \hat{\vx}^{n}_{k+1}) \right] \\ 
    &\leq \sqrt{\E_{\vx \sim \zeta^n_{1, k}}\left[h^2(\vx)\right]} 
    \min\left\{\nicefrac{1}{\sigma}\norm{\vf^{*}(\vx^{n}_{k}, \vpi_n(\vx^{n}_{k})) - \vf_n(\vx^{n}_{k}, \vpi_n(\vx^{n}_{k}))}, 1\right\} \\
    &\leq \sqrt{\E_{\vx \sim \zeta^n_{1, k}}\left[h^2(\vx)\right]} (1 + \sqrt{d}_x)\nicefrac{\beta_n}{\sigma} \norm{\vsigma_n(\vx^{n}_k, \vpi_n(\vx^n_k))}. \tag{\cite[Cor. 3]{sukhija2024optimistic}}
\end{align*}
Therefore, we have
\begin{align*}
  &\sum^{N-1}_{n=0} \sum^{H_n-1}_{k=0} \E_{\vx^{n}_{k}, \dots \vx^{0}_{1}|  \vx_0}\left[  \E_{\vw^{n}_{k}}\left[B(\vpi_n, \vf_n, \vx^{n}_{k+1}) - B(\vpi_n, \vf_n, \hat{\vx}^{n}_{k+1}) \right] \right] \\
&\leq \sum^{N-1}_{n=0} \sum^{H_n-1}_{k=0} \E_{\vx^{n}_{k}, \dots \vx^{0}_{1}|  \vx_0}\left[
  \sqrt{\E_{\vx \sim \zeta^n_{1, k}}\left[h^2(\vx)\right]} (1 + \sqrt{d}_x)\nicefrac{\beta_n}{\sigma} \norm{\vsigma_n(\vx^{n}_k, \vpi_n(\vx^n_k))} \right] \\
    &\leq \sum^{N-1}_{n=0} \sum^{H_n-1}_{k=0} (1 + \sqrt{d}_x)\nicefrac{\beta_n}{\sigma} \sqrt{\E_{\vx^{n}_{k}, \dots \vx^{0}_{1}|  \vx_0}\left[ \E_{\vx \sim \zeta^n_{1, k}}\left[h^2(\vx)\right]\right] \E_{\vx^{n}_{k}, \dots \vx^{0}_{1}|  \vx_0}\left[\norm{\vsigma_n(\vx^{n}_k, \vpi_n(\vx^n_k))}^2\right]} \\
    &\leq (1 + \sqrt{d}_x)\nicefrac{\beta_T}{\sigma} \sqrt{\sum^{N-1}_{n=0} \sum^{H_n-1}_{k=0}\E_{\vx^{n}_{k}, \dots \vx^{0}_{1}|  \vx_0}\left[ \E_{\vx \sim \zeta^n_{1, k}}\left[h^2(\vx)\right]\right]} \\ 
    &\times \sqrt{\sum^{N-1}_{n=0} \sum^{H_n-1}_{k=0}\E_{\vx^{n}_{k}, \dots \vx^{0}_{1}|  \vx_0}\left[\norm{\vsigma_n(\vx^{n}_k, \vpi_n(\vx^n_k))}^2\right]}
\end{align*}
Here, for the second and third inequality, we use Cauchy-Schwarz. 
Now we bound the two terms above individually.

First we bound $\E_{\vx \sim \zeta^n_{1, k}}\left[h^2(\vx)\right]$.
\begin{align*}
    \E_{\vx \sim \zeta^n_{1, k}}\left[h^2(\vx)\right] &=  \E_{\vx \sim \zeta^n_{1, k}}\left[(B(\vpi_n, \vf_n, \vx) - \Bar{B})^2\right] \\
    &= \E_{\vx \sim \zeta^n_{1, k}}\left[(B(\vpi_n, \vf_n, \vx) - \E_{\vx \sim \zeta^n_{2, k}} \left[B(\vpi_n, \vf_n, \vx)\right])^2\right] \\
    &\leq \left(\frac{C_2}{1-\hat{\lambda}}\right)^2 \E_{\vx \sim \zeta^n_{1, k}}\left[(2 + V^{\vpi_n}(\vx) + \E_{\vx \sim \zeta^n_{2, k}} \left[V^{\vpi_n}(\vx)\right])^2\right] \tag{\cref{thm: existence of average cost problem hallucinated system}} \\
    &\leq \left(\frac{C_2}{1-\hat{\lambda}}\right)^2 \E_{\vx \sim \zeta^n_{1, k}}\left[(2 + V^{\vpi_n}(\vx) + \gamma V^{\vpi_n}(\vx^n_{k}) + \hat{K})^2\right] \tag{\cref{lem: stability of hallucinated system}} \\
    &\leq \left(\frac{\sqrt{2}C_2}{1-\hat{\lambda}}\right)^2 \E_{\vx \sim \zeta^n_{1, k}}\left[(V^{\vpi_n}(\vx))^2 + (2 + \gamma V^{\vpi_n}(\vx^n_{k}) + \hat{K})^2\right] \\
    &\leq \left(\frac{\sqrt{2}C_2}{1-\hat{\lambda}}\right)^2 \left( \E_{\vx^n_{k+1}| \vx^n_{k}}\left[(V^{\vpi_n}(\vx_{k+1}))^2\right] + 2\gamma^2 (V^{\vpi_n}(\vx^n_{k}))^2 + 2 (2 + \hat{K})^2 \right) \\
\end{align*}
Furthermore, we have from \cref{lem:bounded_second_moment}.
\begin{align*}
   &\E_{\vx^{n}_{k}, \dots \vx^{0}_{1}|  \vx_0} \left[ \E_{\vx^n_{k+1}| \vx^n_{k}}\left[(V^{\vpi_n}(\vx_{k+1}))^2\right] + 2\gamma^2 (V^{\vpi_n}(\vx^n_{k}))^2 \right] \\
   &= \E_{\vx^{n}_{k+1}, \dots \vx^{0}_{1}|  \vx_0} \left[(V^{\vpi_n}(\vx_{k+1}))^2\right] + 2\gamma^2 \E_{\vx^{n}_{k}, \dots \vx^{0}_{1}|  \vx_0} \left[(V^{\vpi_n}(\vx^{n}_{k}))^2\right]
   \leq (1 + 2\gamma^2)D_3(\vx_0, K, \gamma, \nu).
\end{align*}
In the end, we get 
\begin{align*}
    &\sqrt{\sum^{N-1}_{n=0} \sum^{H_n-1}_{k=0}\E_{\vx^{n}_{k}, \dots \vx^{0}_{1}|  \vx_0}\left[ \E_{\vx \sim \zeta^n_{1, k}}\left[h^2(\vx)\right]\right]} \\
    &\leq \left(\frac{\sqrt{2}C_2}{1-\hat{\lambda}}\right) \sqrt{\sum^{N-1}_{n=0} \sum^{H_n-1}_{k=0} (1 + 2\gamma^2)D_3(\vx_0, K, \gamma, \nu) +  2 (2 + \hat{K})^2} \\
    &= \left(\frac{\sqrt{2}C_2}{1-\hat{\lambda}}\right) \sqrt{(1 + 2\gamma^2)D_3(\vx_0, K, \gamma, \nu) +  2 (2 + \hat{K})^2} \sqrt{\sum^{N-1}_{n=0} H_n} \\
    &= \left(\frac{\sqrt{2}C_2}{1-\hat{\lambda}}\right) \sqrt{(1 + 2\gamma^2)D_3(\vx_0, K, \gamma, \nu) +  2 (2 + \hat{K})^2} \sqrt{T}.
\end{align*}

Next, we apply \Cref{lemma: sum of uncertainties bound} for the second term.
\begin{equation*}
    \sqrt{\sum^{N-1}_{n=0} \sum^{H_n-1}_{k=0} \E_{\vx^{n}_{k}, \dots \vx^{0}_{1}|  \vx_0}\left[\norm{\vsigma_n(\vx^{n}_k, \vpi_n(\vx^n_k))}^2\right]} 
    \leq C'\sqrt{\Gamma_T} 
\end{equation*}
Here $\Gamma_T$ is the maximum information gain.

If we set $D_4(\vx_0, K, \gamma) = \frac{C'(1 + \sqrt{d_x})}{\sigma}\left(\frac{\sqrt{2}C_2}{1-\hat{\lambda}}\right) \sqrt{(1 + 2\gamma^2)D_3(\vx_0, K, \gamma, \nu) +  2 (2 + \hat{K})^2}$, we have
\begin{align*}
    &\sum^{N-1}_{n=0} \sum^{H_n-1}_{k=0} \E_{\vx^{n}_{k}, \dots \vx^{0}_{1}|  \vx_0}\left[  \E_{\vw^{n}_{k}}\left[B(\vpi_n, \vf_n, \vx^{n}_{k+1}) - B(\vpi_n, \vf_n, \hat{\vx}^{n}_{k+1}) \right] \right] \\ 
    &\leq (1 + \sqrt{d}_x)\nicefrac{\beta_T}{\sigma} \sqrt{\sum^{N-1}_{n=0} \sum^{H_n-1}_{k=0}\E_{\vx^{n}_{k}, \dots \vx^{0}_{1}|  \vx_0}\left[ \E_{\vx \sim \zeta^n_{1, k}}\left[h^2(\vx)\right]\right]} \\
    &\times \sqrt{\sum^{N-1}_{n=0} \sum^{H_n-1}_{k=0}\E_{\vx^{n}_{k}, \dots \vx^{0}_{1}|  \vx_0}\left[\norm{\vsigma_n(\vx^{n}_k, \vpi_n(\vx^n_k))}^2\right]} \\
    &\leq (1 + \sqrt{d}_x)\nicefrac{\beta_T}{\sigma}\left(\frac{\sqrt{2}C_2}{1-\hat{\lambda}}\right) \sqrt{(1 + 2\gamma^2)D_3(\vx_0, K, \gamma, \nu) +  2 (2 + \hat{K})^2} \sqrt{T} C'\sqrt{\Gamma_T} \\
    &\leq D_4(\vx_0, K, \gamma) \beta_T \sqrt{T \Gamma_T}
\end{align*}

\textbf{Proof for (B)}:
\begin{align*}
    \sum^{N-1}_{n=0} &\sum^{H_n-1}_{k=0} \E\left[ B(\vpi, \vf_n, \vx^{n}_{k}) - B(\vpi, \vf_n, \vx^{n}_{k+1})\right] =  \sum^{N-1}_{n=0} \E\left[B(\vpi, \vf_n, \vx^{n}_{0}) - B(\vpi, \vf_n, \vx^{n}_{H_n})\right] \\
    &\leq \frac{C_2}{1-\hat{\lambda}} \sum^{N-1}_{n=0} \left(2 + \E\left[V^{\vpi}(\vx^{n}_{0})  + V^{\vpi}(\vx^{n}_{H_n})\right]\right) \tag{\cref{thm: existence of average cost problem hallucinated system}}\\
    &\leq \frac{2C_2}{1-\hat{\lambda}} \sum^{N-1}_{n=0} \left(1 + D(\vx_0, K, \gamma) \right) \tag{\cref{lem:bounded first moment}}\\
    &= \frac{2C_2}{1-\hat{\lambda}} (1 + D(\vx_0, K, \gamma) ) N \\
    &= D_5(\vx_0, K, \gamma) N.
        \end{align*}
        Here $D_5(\vx_0, K, \gamma) = \frac{2C_2}{1-\hat{\lambda}} (1 + D(\vx_0, K, \gamma) )$.
From \Cref{lemma: number of statistical model computations} follows that $N < \frac{1}{\bar{K}_I}\Gamma_T$
To this end, we get for our regret
\begin{align*}
    R_T &= 
        \E\left[\sum^{N-1}_{n=0} \sum^{H_n-1}_{k=0} c(\vx^{n}_k, \vpi_n(\vx^{n}_k)) - A(\vpi^*)\right] \\
        &\leq D_4(\vx_0, K, \gamma) \beta_T \sqrt{T \Gamma_T} + D_5(\vx_0, K, \gamma) N \\
    & \leq D_4(\vx_0, K, \gamma) \beta_T \sqrt{T \Gamma_T} + D_5(\vx_0, K, \gamma) \frac{1}{\bar{K}_I}\Gamma_T
\end{align*}
    \end{proof}
This regret is sublinear for a very rich class of functions. 
We summarize bounds on $\Gamma_T$ from \cite{vakili2021information} in \Cref{table: gamma magnitude bounds for different kernels}.
Furthermore, note that $D_4(\vx_0, K, \gamma) \in (0, \infty)$ for all $\vx_0 \in \setX$ with $\norm{\vx_0} < \infty$, $K < \infty$, $\gamma \in (0, 1)$. The same holds for $D_5(\vx_0, K, \gamma)$. Moreover, since $V^{\vpi}(\vx)$ is $\Theta(\zeta(\norm{\vx}))$, both $D_4$ and $D_5$ are $\Theta(\zeta(\norm{\vx_0}))$. 

   \begin{table}[ht!]
\begin{center}
\caption{Maximum information gain bounds for common choice of kernels.}
\label{table: gamma magnitude bounds for different kernels}
\begin{tabular}{@{}lll@{}}
\toprule
Kernel&$k(\vx, \vx')$ & $\Gamma_T$ \\ \midrule
    Linear &$\vx^\top \vx'$   & $\mathcal{O}\left(d \log(T)\right)$                    \\
    RBF &$e^{-\frac{\norm{\vx - \vx'}^2}{2l^2}}$& $\mathcal{O}\left( \log^{d+1}(T)\right)$                    \\
    Matèrn &$\frac{1}{\Gamma(\nu)2^{\nu - 1}}\left(\frac{\sqrt{2\nu}\norm{\vx-\vx'}}{l}\right)^{\nu}B_{\nu}\left(\frac{\sqrt{2\nu}\norm{\vx-\vx'}}{l}\right)$  & $\mathcal{O}\left(T^{\frac{d}{2\nu + d}}\log^{\frac{2\nu}{2\nu+d}}(T)\right)$                     \\ \bottomrule
\end{tabular}
\end{center}
\end{table}

%% file: backmatter/experimental_details.tex
\section{Practical algorithm and Experimental Details}\label{sec: Experimental details}
\begin{algorithm}[t]
    \caption{Practical \textbf{\neorl:}}
    \begin{algorithmic}[]
        \STATE {\textbf{Init:}}{ Aleatoric uncertainty $\sigma$, Probability $\delta$, Statistical model $(\vmu_0, \vsigma_0, \beta_0(\delta))$}
        \FOR{$n=1, \ldots, N$}{
            \FOR{$h=1, \ldots, H$}{
            \vspace{-0.5cm}
            \STATE {
            \begin{align*}
                &\min_{\vu_{0:H_{\text{MPC}}-1}, \veta_{0;H_{\text{MPC}}-1}} \E \left[\sum_{h=0}^{H_{\text{MPC}}-1} c(\vhx_h, \vu_h) \right]; \vx_0 = \vx^n_h &&\text{\ding{228} Solve MPC problem}  \\
                  &(\vx^{h}_n, \vu^*_0, \vx^{h+1}_n) \leftarrow \textsc{Rollout}(\vu^*_0) \quad &&\text{\ding{228} Collect transition} 
                \end{align*}
                }
              }
                      \ENDFOR
                           \STATE{
                Update $(\vmu_n, \vsigma_n, \beta_n) \leftarrow \setD_n$
              }
              }
        \ENDFOR
    \end{algorithmic}
    \label{alg:neorl_pratical}
\end{algorithm}

In this section, we provide the practical algorithm~\cref{alg:neorl_pratical}, provide all hyperparameters used in our experiments in \cref{tab:environment_hyperparams}, and the cost function for the environments. All our experiments within 1-8 hours\footnote{based on the environment} on a GPU (NVIDIA GeForce RTX 2080 Ti). For \neorl, we use $\beta_n=2$ for all the experiments, except for the Swimmer and the SoftArm environment where we use $\beta_n=1$.
\begin{table}[ht]
\centering
    \caption{Hyperparameters for results in~\Cref{sec:experiments}.}
    \label{tab:environment_hyperparams}
\begin{adjustbox}{max width=\linewidth}\begin{threeparttable}
\begin{tabular}{c|ccccc|ccccc|cc}
\midrule
Environment & \multicolumn{5}{c|}{iCEM parameters}                                                                                                                                                                                                                                        & \multicolumn{5}{c|}{Model training parameters}                                                                                                                                         & \multicolumn{1}{l}{} & \multicolumn{1}{l}{}                                    \\
            & 
            \begin{tabular}[c]{@{}c@{}}
            
            Number of \\ samples\end{tabular} & \begin{tabular}[c]{@{}c@{}}Number of \\ elites\end{tabular} & \begin{tabular}[c]{@{}c@{}}Optimizer\\ steps\end{tabular} & $H_{\text{MPC}}$ & Particles & \begin{tabular}[c]{@{}c@{}}Number of \\ ensembles\end{tabular} & \begin{tabular}[c]{@{}c@{}}Network\\ architecture\end{tabular} & Learning rate & Batch size & \begin{tabular}[c]{@{}c@{}}Number of\\ epochs\end{tabular} & H                    & \begin{tabular}[c]{@{}c@{}}Action\\ Repeat\end{tabular} \\ \midrule
Pendulum-GP & 500                                                          & 50                                                          & 10                                                        & 20      & 5                                                             & -                                                              & -                                                              & 0.01          & 64         & -                                                          & 10                   & 1                                                       \\
Pendulum    & 500                                                          & 50                                                          & 10                                                        & 20      & 5                                                              & 10                                                             & $256 \times 2$                                                     & 0.001         & 64         & 50                                                         & 10                   & 1                                                       \\
MountainCar & 1000                                                         & 100                                                         & 5                                                         & 50      & 5                                                              & 10                                                             & $256 \times 2$                                                     & 0.001         & 64         & 50                                                         & 10                   & 2                                                       \\
Reacher     & 1000                                                         & 100                                                         & 10                                                        & 50      & 5                                                                & 10                                                             & $256 \times 2$                                                    & 0.001         & 64         & 50                                                         & 10                   & 2                                                       \\
CartPole    & 1000                                                         & 100                                                         & 10                                                        & 50      & 5                                                                & 10                                                             & $256 \times 2$                                                     & 0.001         & 64         & 50                                                         & 10                   & 2                                                       \\
Swimmer     & 500                                                          & 50                                                          & 10                                                        & 30      & 5                                                              & 10                                                             & $256 \times 4$  & 0.00005       & 64         & 100                                                        & 200                  & 4                                                       \\
SoftArm     & 500                                                          & 50                                                          & 10                                                        & 20      & 5                                                             & 10                                                             & $256 \times 4$  & 0.00005       & 64         & 50                                                         & 20                   & 1                                                       \\
RaceCar     & 1000                                                         & 100                                                         & 10                                                        & 50      & 5                                                            & 10                                                             & $256 \times 2$                                                      & 0.001         & 64         & 50                                                         & 10                   & 1                                                      
\end{tabular}
\end{threeparttable}\end{adjustbox}
\end{table}

\begin{table}[ht]
    \centering
    \caption{Cost function for the environments presented in \Cref{sec:experiments}.}
    \label{tab:environment_rewards}
\begin{tabular}{cc}
\midrule
\textbf{Environment}                   & \textbf{Cost} $c(\vx_t, \vu_t)$                                                                                                      \\
\midrule
Pendulum       & $\theta_t^2 + 0.1 \Dot{\theta}_t + 0.1 u_t^2$                      \\
MountainCar                   & $0.1 u_t^2  + 100 (1\{\vx_t \not\in \vx_{\text{goal}}\})$                            \\
Reacher       & 
$\norm{\vx_t - \vx_{\text{target}}} + 0.1 \norm{u_t} $ \\
CartPole      & $\norm{\vx^{\text{pos}}_t - \vx^{\text{pos}}_{\text{target}}}^2 + 10 (\cos(\theta_t) - 1)^2 + 0.2 \norm{u_t}^2$                                                                                     \\
Swimmer & $\norm{\vx_t - \vx_{\text{target}}}$                                                                                      \\
SoftArm & $\norm{\vx_t - \vx_{\text{target}}}$                                                                                      \\
RaceCar         & $\norm{\vx_t - \vx_{\text{target}}}$                                \\
\hline \\
\end{tabular}
\end{table}

%% file: main.bbl
\begin{thebibliography}{73}
\providecommand{\natexlab}[1]{#1}
\providecommand{\url}[1]{\texttt{#1}}
\expandafter\ifx\csname urlstyle\endcsname\relax
  \providecommand{\doi}[1]{doi: #1}\else
  \providecommand{\doi}{doi: \begingroup \urlstyle{rm}\Url}\fi

\bibitem[Abbasi-Yadkori \& Szepesv{\'a}ri(2011)Abbasi-Yadkori and Szepesv{\'a}ri]{abbasi2011regret}
Abbasi-Yadkori, Y. and Szepesv{\'a}ri, C.
\newblock Regret bounds for the adaptive control of linear quadratic systems.
\newblock In \emph{Conference on Learning Theory}, 2011.

\bibitem[Abeille \& Lazaric(2020)Abeille and Lazaric]{abeille2020efficient}
Abeille, M. and Lazaric, A.
\newblock Efficient optimistic exploration in linear-quadratic regulators via lagrangian relaxation.
\newblock In \emph{International Conference on Machine Learning}, 2020.

\bibitem[Anderson \& Moore(2007)Anderson and Moore]{anderson2007optimal}
Anderson, B.~D. and Moore, J.~B.
\newblock \emph{Optimal control: linear quadratic methods}.
\newblock Courier Corporation, 2007.

\bibitem[Annaswamy(2023)]{annaswamy2023adaptive}
Annaswamy, A.~M.
\newblock Adaptive control and intersections with reinforcement learning.
\newblock \emph{Annual Review of Control, Robotics, and Autonomous Systems}, 2023.

\bibitem[Arapostathis et~al.(1993)Arapostathis, Borkar, Fern\'{a}ndez-Gaucherand, Ghosh, and Marcus]{average_cost_survey}
Arapostathis, A., Borkar, V.~S., Fern\'{a}ndez-Gaucherand, E., Ghosh, M.~K., and Marcus, S.~I.
\newblock Discrete-time controlled markov processes with average cost criterion: A survey.
\newblock \emph{SIAM Journal on Control and Optimization}, 1993.

\bibitem[{\AA}str{\"o}m \& Wittenmark(2013){\AA}str{\"o}m and Wittenmark]{adaptivecontrol}
{\AA}str{\"o}m, K.~J. and Wittenmark, B.
\newblock \emph{Adaptive Control}.
\newblock Courier Corporation, 2013.

\bibitem[Bartlett \& Tewari(2012)Bartlett and Tewari]{bartlett2012regal}
Bartlett, P.~L. and Tewari, A.
\newblock Regal: A regularization based algorithm for reinforcement learning in weakly communicating mdps.
\newblock \emph{arXiv preprint arXiv:1205.2661}, 2012.

\bibitem[Berberich \& Allgöwer(2024)Berberich and Allgöwer]{berberich2024overviewsystemstheoreticguaranteesdatadriven}
Berberich, J. and Allgöwer, F.
\newblock An overview of systems-theoretic guarantees in data-driven model predictive control, 2024.
\newblock URL \url{https://arxiv.org/abs/2406.04130}.

\bibitem[Brafman \& Tennenholtz(2002)Brafman and Tennenholtz]{brafman2002r}
Brafman, R.~I. and Tennenholtz, M.
\newblock R-max-a general polynomial time algorithm for near-optimal reinforcement learning.
\newblock \emph{Journal of Machine Learning Research}, 2002.

\bibitem[Brockman et~al.(2016)Brockman, Cheung, Pettersson, Schneider, Schulman, Tang, and Zaremba]{brockman2016openai}
Brockman, G., Cheung, V., Pettersson, L., Schneider, J., Schulman, J., Tang, J., and Zaremba, W.
\newblock Openai gym.
\newblock \emph{arXiv preprint arXiv:1606.01540}, 2016.

\bibitem[Chowdhury \& Gopalan(2017)Chowdhury and Gopalan]{chowdhury2017kernelized}
Chowdhury, S.~R. and Gopalan, A.
\newblock On kernelized multi-armed bandits.
\newblock In \emph{ICML}, 2017.

\bibitem[Chua et~al.(2018)Chua, Calandra, McAllister, and Levine]{chua2018pets}
Chua, K., Calandra, R., McAllister, R., and Levine, S.
\newblock Deep reinforcement learning in a handful of trials using probabilistic dynamics models.
\newblock In \emph{NeurIPS}, 2018.

\bibitem[Cohen et~al.(2019)Cohen, Koren, and Mansour]{cohen2019learning}
Cohen, A., Koren, T., and Mansour, Y.
\newblock Learning linear-quadratic regulators efficiently with only $\sqrt{T}$ regret.
\newblock In \emph{International Conference on Machine Learning}, 2019.

\bibitem[Curi et~al.(2020)Curi, Berkenkamp, and Krause]{curi2020efficient}
Curi, S., Berkenkamp, F., and Krause, A.
\newblock Efficient model-based reinforcement learning through optimistic policy search and planning.
\newblock \emph{NeurIPS}, 33:\penalty0 14156--14170, 2020.

\bibitem[Dean et~al.(2020)Dean, Mania, Matni, Recht, and Tu]{dean2020sample}
Dean, S., Mania, H., Matni, N., Recht, B., and Tu, S.
\newblock On the sample complexity of the linear quadratic regulator.
\newblock \emph{Foundations of Computational Mathematics}, 20\penalty0 (4):\penalty0 633--679, 2020.

\bibitem[Dewanto et~al.(2020)Dewanto, Dunn, Eshragh, Gallagher, and Roosta]{dewanto2020average}
Dewanto, V., Dunn, G., Eshragh, A., Gallagher, M., and Roosta, F.
\newblock Average-reward model-free reinforcement learning: a systematic review and literature mapping.
\newblock \emph{arXiv preprint arXiv:2010.08920}, 2020.

\bibitem[Eysenbach et~al.(2018)Eysenbach, Gu, Ibarz, and Levine]{eysenbach2017leave}
Eysenbach, B., Gu, S., Ibarz, J., and Levine, S.
\newblock Leave no trace: Learning to reset for safe and autonomous reinforcement learning.
\newblock \emph{International Conference on Learning Representations}, 2018.

\bibitem[Faradonbeh et~al.(2020)Faradonbeh, Tewari, and Michailidis]{faradonbeh2020optimism}
Faradonbeh, M. K.~S., Tewari, A., and Michailidis, G.
\newblock Optimism-based adaptive regulation of linear-quadratic systems.
\newblock \emph{IEEE Transactions on Automatic Control}, 2020.

\bibitem[Foster et~al.(2020)Foster, Sarkar, and Rakhlin]{foster2020learning}
Foster, D., Sarkar, T., and Rakhlin, A.
\newblock Learning nonlinear dynamical systems from a single trajectory.
\newblock In \emph{Learning for Dynamics and Control}, 2020.

\bibitem[García et~al.(1989)García, Prett, and Morari]{mpc}
García, C.~E., Prett, D.~M., and Morari, M.
\newblock Model predictive control: Theory and practice - a survey.
\newblock \emph{Automatica}, pp.\  335--348, 1989.

\bibitem[Ha \& Schmidhuber(2018)Ha and Schmidhuber]{ha2018recurrent}
Ha, D. and Schmidhuber, J.
\newblock Recurrent world models facilitate policy evolution.
\newblock \emph{Advances in neural information processing systems}, 31, 2018.

\bibitem[Hafner et~al.(2019)Hafner, Lillicrap, Fischer, Villegas, Ha, Lee, and Davidson]{hafner2019learning}
Hafner, D., Lillicrap, T., Fischer, I., Villegas, R., Ha, D., Lee, H., and Davidson, J.
\newblock Learning latent dynamics for planning from pixels.
\newblock In \emph{International conference on machine learning}, 2019.

\bibitem[Hafner et~al.(2020)Hafner, Lillicrap, Ba, and Norouzi]{hafner2019dream}
Hafner, D., Lillicrap, T., Ba, J., and Norouzi, M.
\newblock Dream to control: Learning behaviors by latent imagination.
\newblock \emph{ICLR}, 2020.

\bibitem[Hairer \& Mattingly(2011)Hairer and Mattingly]{hairer2011yet}
Hairer, M. and Mattingly, J.~C.
\newblock Yet another look at harris’ ergodic theorem for markov chains.
\newblock In \emph{Seminar on Stochastic Analysis, Random Fields and Applications VI: Centro Stefano Franscini, Ascona, May 2008}, pp.\  109--117. Springer, 2011.

\bibitem[Han et~al.(2015)Han, Levine, and Abbeel]{han2015learning}
Han, W., Levine, S., and Abbeel, P.
\newblock Learning compound multi-step controllers under unknown dynamics.
\newblock In \emph{Intelligent Robots and Systems (IROS)}, 2015.

\bibitem[Jaksch et~al.(2010)Jaksch, Ortner, and Auer]{jaksch10a}
Jaksch, T., Ortner, R., and Auer, P.
\newblock Near-optimal regret bounds for reinforcement learning.
\newblock \emph{Journal of Machine Learning Research}, 2010.

\bibitem[Janner et~al.(2019)Janner, Fu, Zhang, and Levine]{janner2019trust}
Janner, M., Fu, J., Zhang, M., and Levine, S.
\newblock When to trust your model: Model-based policy optimization.
\newblock \emph{Advances in neural information processing systems}, 2019.

\bibitem[Kabzan et~al.(2020)Kabzan, Valls, Reijgwart, Hendrikx, Ehmke, Prajapat, B{\"u}hler, Gosala, Gupta, Sivanesan, et~al.]{kabzan2020amz}
Kabzan, J., Valls, M.~I., Reijgwart, V.~J., Hendrikx, H.~F., Ehmke, C., Prajapat, M., B{\"u}hler, A., Gosala, N., Gupta, M., Sivanesan, R., et~al.
\newblock Amz driverless: The full autonomous racing system.
\newblock \emph{Journal of Field Robotics}, 2020.

\bibitem[Kakade et~al.(2020)Kakade, Krishnamurthy, Lowrey, Ohnishi, and Sun]{kakade2020information}
Kakade, S., Krishnamurthy, A., Lowrey, K., Ohnishi, M., and Sun, W.
\newblock Information theoretic regret bounds for online nonlinear control.
\newblock \emph{NeurIPS}, 33:\penalty0 15312--15325, 2020.

\bibitem[Kakade(2003)]{kakade2003sample}
Kakade, S.~M.
\newblock \emph{On the sample complexity of reinforcement learning}.
\newblock University of London, University College London (United Kingdom), 2003.

\bibitem[Kearns \& Singh(2002)Kearns and Singh]{kearns2002near}
Kearns, M. and Singh, S.
\newblock Near-optimal reinforcement learning in polynomial time.
\newblock \emph{Machine learning}, 2002.

\bibitem[Khalil(2015)]{khalil2015nonlinear}
Khalil, H.~K.
\newblock \emph{Nonlinear control}, volume 406.
\newblock Pearson New York, 2015.

\bibitem[Kipf et~al.(2019)Kipf, Van~der Pol, and Welling]{kipf2019contrastive}
Kipf, T., Van~der Pol, E., and Welling, M.
\newblock Contrastive learning of structured world models.
\newblock \emph{arXiv preprint arXiv:1911.12247}, 2019.

\bibitem[Krsti{\'c} et~al.(1992)Krsti{\'c}, Kanellakopoulos, and Kokotovi{\'c}]{krstic1992adaptive}
Krsti{\'c}, M., Kanellakopoulos, I., and Kokotovi{\'c}, P.
\newblock Adaptive nonlinear control without overparametrization.
\newblock \emph{Systems \& Control Letters}, 1992.

\bibitem[Krsti{\'c} et~al.(1995)Krsti{\'c}, Kokotovic, and Kanellakopoulos]{krstic1995nonlinear}
Krsti{\'c}, M., Kokotovic, P.~V., and Kanellakopoulos, I.
\newblock \emph{Nonlinear and adaptive control design}.
\newblock John Wiley \& Sons, Inc., 1995.

\bibitem[Kuleshov et~al.(2018)Kuleshov, Fenner, and Ermon]{kuleshov2018accurate}
Kuleshov, V., Fenner, N., and Ermon, S.
\newblock Accurate uncertainties for deep learning using calibrated regression.
\newblock In \emph{ICML}, pp.\  2796--2804. PMLR, 2018.

\bibitem[Lai \& Wei(1982)Lai and Wei]{lai1982least}
Lai, T.~L. and Wei, C.~Z.
\newblock Least squares estimates in stochastic regression models with applications to identification and control of dynamic systems.
\newblock \emph{The Annals of Statistics}, 1982.

\bibitem[Lai \& Wei(1987)Lai and Wei]{lai1987asymptotically}
Lai, T.~L. and Wei, C.-Z.
\newblock Asymptotically efficient self-tuning regulators.
\newblock \emph{SIAM Journal on Control and Optimization}, 1987.

\bibitem[Lakshminarayanan et~al.(2017)Lakshminarayanan, Pritzel, and Blundell]{lakshminarayanan2017simple}
Lakshminarayanan, B., Pritzel, A., and Blundell, C.
\newblock Simple and scalable predictive uncertainty estimation using deep ensembles, 2017.

\bibitem[Lale et~al.(2020)Lale, Azizzadenesheli, Hassibi, and Anandkumar]{lale2020logarithmic}
Lale, S., Azizzadenesheli, K., Hassibi, B., and Anandkumar, A.
\newblock Logarithmic regret bound in partially observable linear dynamical systems.
\newblock \emph{Advances in Neural Information Processing Systems}, 2020.

\bibitem[Lale et~al.(2021)Lale, Azizzadenesheli, Hassibi, and Anandkumar]{lale2021model}
Lale, S., Azizzadenesheli, K., Hassibi, B., and Anandkumar, A.
\newblock Model learning predictive control in nonlinear dynamical systems.
\newblock In \emph{Conference on Decision and Control (CDC)}. IEEE, 2021.

\bibitem[Ma et~al.(2021)Ma, Tang, Xia, Yang, and Zhao]{ma2021average}
Ma, X., Tang, X., Xia, L., Yang, J., and Zhao, Q.
\newblock Average-reward reinforcement learning with trust region methods.
\newblock \emph{International Joint Conference on Artificial Intelligence}, 2021.

\bibitem[Mahadevan(1996)]{mahadevan1996average}
Mahadevan, S.
\newblock Average reward reinforcement learning: Foundations, algorithms, and empirical results.
\newblock \emph{Machine learning}, 1996.

\bibitem[Mania et~al.(2020)Mania, Jordan, and Recht]{mania2020active}
Mania, H., Jordan, M.~I., and Recht, B.
\newblock Active learning for nonlinear system identification with guarantees.
\newblock \emph{arXiv preprint arXiv:2006.10277}, 2020.

\bibitem[Meyn \& Tweedie(2012)Meyn and Tweedie]{meyn2012markov}
Meyn, S.~P. and Tweedie, R.~L.
\newblock \emph{Markov chains and stochastic stability}.
\newblock Springer Science \& Business Media, 2012.

\bibitem[Osband \& Van~Roy(2017)Osband and Van~Roy]{osband2017posterior}
Osband, I. and Van~Roy, B.
\newblock Why is posterior sampling better than optimism for reinforcement learning?
\newblock In \emph{International conference on machine learning}, 2017.

\bibitem[Ouyang et~al.(2017)Ouyang, Gagrani, Nayyar, and Jain]{ouyang2017learning}
Ouyang, Y., Gagrani, M., Nayyar, A., and Jain, R.
\newblock Learning unknown markov decision processes: A thompson sampling approach.
\newblock \emph{Advances in neural information processing systems}, 30, 2017.

\bibitem[Pinneri et~al.(2021)Pinneri, Sawant, Blaes, Achterhold, Stueckler, Rolinek, and Martius]{iCem}
Pinneri, C., Sawant, S., Blaes, S., Achterhold, J., Stueckler, J., Rolinek, M., and Martius, G.
\newblock Sample-efficient cross-entropy method for real-time planning.
\newblock In \emph{CORL}, Proceedings of Machine Learning Research, pp.\  1049--1065, 2021.

\bibitem[Puterman(2014)]{puterman2014markov}
Puterman, M.~L.
\newblock \emph{Markov decision processes: discrete stochastic dynamic programming}.
\newblock John Wiley \& Sons, 2014.

\bibitem[Rahimi \& Recht(2007)Rahimi and Recht]{rahimi2007random}
Rahimi, A. and Recht, B.
\newblock Random features for large-scale kernel machines.
\newblock \emph{Advances in neural information processing systems}, 20, 2007.

\bibitem[Rothfuss et~al.(2023)Rothfuss, Sukhija, Birchler, Kassraie, and Krause]{rothfuss2023hallucinated}
Rothfuss, J., Sukhija, B., Birchler, T., Kassraie, P., and Krause, A.
\newblock Hallucinated adversarial control for conservative offline policy evaluation.
\newblock \emph{UAI}, 2023.

\bibitem[Sattar \& Oymak(2022)Sattar and Oymak]{sattar2022non}
Sattar, Y. and Oymak, S.
\newblock Non-asymptotic and accurate learning of nonlinear dynamical systems.
\newblock \emph{Journal of Machine Learning Research}, 2022.

\bibitem[Saxena et~al.(2023)Saxena, Khastagir, Kolathaya, and Bhatnagar]{saxena2023off}
Saxena, N., Khastagir, S., Kolathaya, S., and Bhatnagar, S.
\newblock Off-policy average reward actor-critic with deterministic policy search.
\newblock In \emph{International Conference on Machine Learning}, 2023.

\bibitem[Scarlett et~al.(2017)Scarlett, Bogunovic, and Cevher]{scarlett17a}
Scarlett, J., Bogunovic, I., and Cevher, V.
\newblock Lower bounds on regret for noisy {G}aussian process bandit optimization.
\newblock In \emph{Conference on Learning Theory}, 2017.

\bibitem[Sharma et~al.(2021{\natexlab{a}})Sharma, Gupta, Levine, Hausman, and Finn]{sharma2021curiculum}
Sharma, A., Gupta, A., Levine, S., Hausman, K., and Finn, C.
\newblock Autonomous reinforcement learning via subgoal curricula.
\newblock \emph{Advances in Neural Information Processing Systems}, 2021{\natexlab{a}}.

\bibitem[Sharma et~al.(2021{\natexlab{b}})Sharma, Xu, Sardana, Gupta, Hausman, Levine, and Finn]{sharma2021autonomous}
Sharma, A., Xu, K., Sardana, N., Gupta, A., Hausman, K., Levine, S., and Finn, C.
\newblock Autonomous reinforcement learning: Formalism and benchmarking.
\newblock \emph{arXiv preprint arXiv:2112.09605}, 2021{\natexlab{b}}.

\bibitem[Sharma et~al.(2022)Sharma, Ahmad, and Finn]{sharma2022state}
Sharma, A., Ahmad, R., and Finn, C.
\newblock A state-distribution matching approach to non-episodic reinforcement learning.
\newblock \emph{International Conference on Machine Learning}, 2022.

\bibitem[Simchowitz \& Foster(2020)Simchowitz and Foster]{simchowitz2020naive}
Simchowitz, M. and Foster, D.
\newblock Naive exploration is optimal for online lqr.
\newblock In \emph{International Conference on Machine Learning}. PMLR, 2020.

\bibitem[Srinivas et~al.(2012)Srinivas, Krause, Kakade, and Seeger]{srinivas}
Srinivas, N., Krause, A., Kakade, S.~M., and Seeger, M.~W.
\newblock Information-theoretic regret bounds for gaussian process optimization in the bandit setting.
\newblock \emph{IEEE Transactions on Information Theory}, 2012.

\bibitem[Sukhija et~al.(2024)Sukhija, Treven, Sancaktar, Blaes, Coros, and Krause]{sukhija2024optimistic}
Sukhija, B., Treven, L., Sancaktar, C., Blaes, S., Coros, S., and Krause, A.
\newblock Optimistic active exploration of dynamical systems.
\newblock \emph{NeurIPS}, 2024.

\bibitem[Sussex et~al.(2023)Sussex, Makarova, and Krause]{sussex2022model}
Sussex, S., Makarova, A., and Krause, A.
\newblock Model-based causal bayesian optimization.
\newblock In \emph{ICLR}, May 2023.

\bibitem[Tassa et~al.(2018)Tassa, Doron, Muldal, Erez, Li, Casas, Budden, Abdolmaleki, Merel, Lefrancq, et~al.]{tassa2018deepmind}
Tassa, Y., Doron, Y., Muldal, A., Erez, T., Li, Y., Casas, D. d.~L., Budden, D., Abdolmaleki, A., Merel, J., Lefrancq, A., et~al.
\newblock Deepmind control suite.
\newblock \emph{arXiv preprint arXiv:1801.00690}, 2018.

\bibitem[Tekinalp et~al.(2024)Tekinalp, Kim, Bhosale, Parthasarathy, Naughton, Albazroun, Joon, Cui, Nasiriziba, Stölzle, Shih, and Gazzola]{arman_tekinalp_2024_10883271}
Tekinalp, A., Kim, S.~H., Bhosale, Y., Parthasarathy, T., Naughton, N., Albazroun, A., Joon, R., Cui, S., Nasiriziba, I., Stölzle, M., Shih, C.-H.~C., and Gazzola, M.
\newblock Gazzolalab/pyelastica: v0.3.2, 2024.

\bibitem[Treven et~al.(2021)Treven, Curi, Mutn{\`y}, and Krause]{treven2021learning}
Treven, L., Curi, S., Mutn{\`y}, M., and Krause, A.
\newblock Learning stabilizing controllers for unstable linear quadratic regulators from a single trajectory.
\newblock In \emph{Learning for Dynamics and Control}, 2021.

\bibitem[Treven et~al.(2024)Treven, Hübotter, Sukhija, Dörfler, and Krause]{treven2024ocorl}
Treven, L., Hübotter, J., Sukhija, B., Dörfler, F., and Krause, A.
\newblock Efficient exploration in continuous-time model-based reinforcement learning.
\newblock \emph{NeurIPS}, 2024.

\bibitem[Vakili et~al.(2021)Vakili, Khezeli, and Picheny]{vakili2021information}
Vakili, S., Khezeli, K., and Picheny, V.
\newblock On information gain and regret bounds in gaussian process bandits.
\newblock In \emph{AISTATS}, 2021.

\bibitem[Wagenmaker et~al.(2023)Wagenmaker, Shi, and Jamieson]{wagenmaker2023optimal}
Wagenmaker, A., Shi, G., and Jamieson, K.
\newblock Optimal exploration for model-based rl in nonlinear systems.
\newblock \emph{arXiv preprint arXiv:2306.09210}, 2023.

\bibitem[Williams et~al.(2017)Williams, Wagener, Goldfain, Drews, Rehg, Boots, and Theodorou]{williams2017information}
Williams, G., Wagener, N., Goldfain, B., Drews, P., Rehg, J.~M., Boots, B., and Theodorou, E.~A.
\newblock Information theoretic mpc for model-based reinforcement learning.
\newblock In \emph{ICRA}, 2017.

\bibitem[Xu et~al.(2020)Xu, Verma, Finn, and Levine]{xu2020continual}
Xu, K., Verma, S., Finn, C., and Levine, S.
\newblock Continual learning of control primitives: Skill discovery via reset-games.
\newblock \emph{Advances in Neural Information Processing Systems}, 2020.

\bibitem[Xu \& Tewari(2020)Xu and Tewari]{xu2020reinforcement}
Xu, Z. and Tewari, A.
\newblock Reinforcement learning in factored mdps: Oracle-efficient algorithms and tighter regret bounds for the non-episodic setting.
\newblock \emph{Advances in Neural Information Processing Systems}, 2020.

\bibitem[Zhang \& Ross(2021)Zhang and Ross]{atrpo}
Zhang, Y. and Ross, K.~W.
\newblock On-policy deep reinforcement learning for the average-reward criterion.
\newblock In \emph{International Conference on Machine Learning}, 2021.

\bibitem[Zhao et~al.(2024)Zhao, D{\"o}rfler, Chiuso, and You]{zhao2024data}
Zhao, F., D{\"o}rfler, F., Chiuso, A., and You, K.
\newblock Data-enabled policy optimization for direct adaptive learning of the lqr.
\newblock \emph{arXiv preprint arXiv:2401.14871}, 2024.

\bibitem[Zhu et~al.(2020)Zhu, Yu, Gupta, Shah, Hartikainen, Singh, Kumar, and Levine]{zhu2020ingredients}
Zhu, H., Yu, J., Gupta, A., Shah, D., Hartikainen, K., Singh, A., Kumar, V., and Levine, S.
\newblock The ingredients of real-world robotic reinforcement learning.
\newblock \emph{arXiv preprint arXiv:2004.12570}, 2020.

\end{thebibliography}
